\documentclass{article}
\usepackage[utf8]{inputenc}

\usepackage{amsmath, amssymb, amsthm, amsfonts}

\newtheorem*{definition}{Definition}

\newtheorem*{example}{Example} 

\newtheorem{theorem}{Theorem}
\newtheorem*{corollary}{Corollary}
\newtheorem*{lemma}{Lemma}
\newtheorem{assumption}{Assumption}

\DeclareMathOperator{\spn}{span}

\DeclareMathOperator*{\argmax}{arg\,max}
\DeclareMathOperator{\rank}{rank}

\usepackage{mathtools}
\usepackage{enumerate}
\usepackage{graphicx}
\usepackage{array}
\usepackage[margin=1in]{geometry}
\usepackage{tikz}
\usepackage{hyperref, doi}
\usepackage{epsf}

\usepackage[utf8]{inputenc} 
\usepackage[T1]{fontenc}    
\usepackage{float}

\usepackage{caption}
\usepackage{subcaption}
\usepackage{array}
\usepackage{algpseudocode,algorithm,algorithmicx}
\usepackage{makecell}

\renewcommand{\algorithmicrequire}{\textbf{Input:}}
\renewcommand{\algorithmicensure}{\textbf{Output:}}

\newcolumntype{?}[1]{!{\vrule width #1}}
\newcommand*\samethanks[1][\value{footnote}]{\footnotemark[#1]}
\DeclarePairedDelimiter\floor{\lfloor}{\rfloor}

\title{Memorization in Overparameterized Autoencoders}

\author{Adityanarayanan Radhakrishnan \thanks{Laboratory for Information \& Decision Systems, and 
 Institute for Data, Systems, and Society, 
 Massachusetts Institute of Technology}
\and Karren Yang \samethanks 
\and Mikhail Belkin \thanks{Department of Computer Science and Engineering, The Ohio State University}
\and Caroline Uhler \samethanks[1]}

\begin{document}
\maketitle 

\begin{abstract}

The ability of deep neural networks to generalize well in the overparameterized regime has become a subject of significant research interest. We show that overparameterized autoencoders exhibit \emph{memorization}, a form of inductive bias that constrains the functions learned through the optimization process to concentrate around the training examples, although the network could in principle represent a much larger function class. In particular, we prove that single-layer fully-connected autoencoders project data onto the (nonlinear) span of the training examples. In addition, we show that deep fully-connected autoencoders learn a map that is \emph{locally contractive} at the training examples, and hence iterating the autoencoder results in convergence to the training examples. Finally, we prove that depth is necessary and provide empirical evidence that it is also sufficient for memorization in convolutional autoencoders. Understanding this inductive bias may shed light on the generalization properties of overparametrized deep neural networks that are currently unexplained by classical statistical theory.  
\end{abstract}

\section{Introduction}

In many practical applications, deep neural networks are trained to achieve near zero training loss, to {\it interpolate} the data. Yet, these networks still show excellent performance on the test data~\cite{RethinkingGeneralization}, a generalization phenomenon which we are only starting to understand~\cite{DoubleDescent}. Indeed, while infinitely many potential interpolating solutions exist, most of them cannot be expected to generalize at all. Thus to understand generalization we need to understand the {\it inductive bias} of these methods, i.e.~the properties of the solution learned by the training procedure. These properties have been explored in a number of recent works including~\cite{brutzkus2017sgd,ImplicitSelfRegularization,neyshabur2014search,savarese2019infinite,soudry2018implicit}.

In this paper, we investigate the inductive bias of overparameterized \emph{autoencoders}~\cite{goodfellow2016deep}, i.e.~maps $f:\mathbb{R}^d \rightarrow \mathbb{R}^d$ that satisfy $f(x^{(i)}) = x^{(i)}$ for $1 \leq i \leq n$ and are obtained by optimizing
\begin{equation*}
\arg\min_{f \in \mathcal{F}} \;\sum\nolimits_{i=1}^n \|f(x^{(i)}) - x^{(i)}\|^2
\end{equation*}
using gradient descent, where $\mathcal{F}$ denotes the network function space. Studying inductive bias in the context of autoencoders is relevant since (1) components of convolutional autoencoders are building blocks of many CNNs;  (2) layerwise pre-training using autoencoders is a standard technique to initialize individual layers of CNNs to improve training~\cite{LayerwiseImagenet, PretrainingLayerwise, UnsupervisedPretraining}; and (3) autoencoder architectures are used in many image-to-image tasks such as image segmentation or impainting~\cite{DeepImagePrior}. Furthermore, the inductive bias that we characterize in autoencoders may apply to more general~architectures.  

While there are many solutions that can interpolate the training examples to achieve zero training loss (see Figure \ref{fig:SchematicAutoencoders}a), in this paper we show that the solutions learned by gradient descent exhibit the following inductive bias:
\vspace{-0.1cm}
\begin{enumerate}
    \item For a single layer fully connected autoencoder, the learned solution maps any input to the ``nonlinear'' span (see Definition~\ref{def:Nonlinear Span}) of the training examples (see Section~\ref{sec:SingleFCMemorization}).
    \item For a multi-layer fully connected autoencoder, the learned solution is \emph{locally contractive} at the training examples, and hence iterating the autoencoder for any input results in convergence to a training example (see Figure \ref{fig:SchematicAutoencoders}a, top right and bottom left, and Section~\ref{sec:ContractiveMaps}). Larger networks result in faster contraction to training examples (Figure \ref{fig:SchematicAutoencoders}a, bottom right).
    \item Subject to sufficient depth, convolutional neural networks exhibit the same properties as fully connected networks (see Figure~\ref{fig:SchematicAutoencoders}b and Section~\ref{sec:ConvolutionalAutoencoders}).
\end{enumerate}

Taken together, our results indicate that overparameterized autoencoders exhibit a form of data-dependent self-regularization that encourages solutions that concentrate around the training examples. {\bf We refer to this form of inductive bias as \emph{memorization}, since it enables recovering training examples from the network.} Furthermore, we show that memorization is robust to early stopping and in general does not imply overfitting, i.e.~the learned solution can be arbitrarily close to the identity function while still being locally contractive at the training examples (see Section~\ref{sec:RobustnessOfMemorization}).



 \begin{figure}[!t]
   \begin{subfigure}[t]{0.4\textwidth}
        \centering
        \includegraphics[scale=0.25]{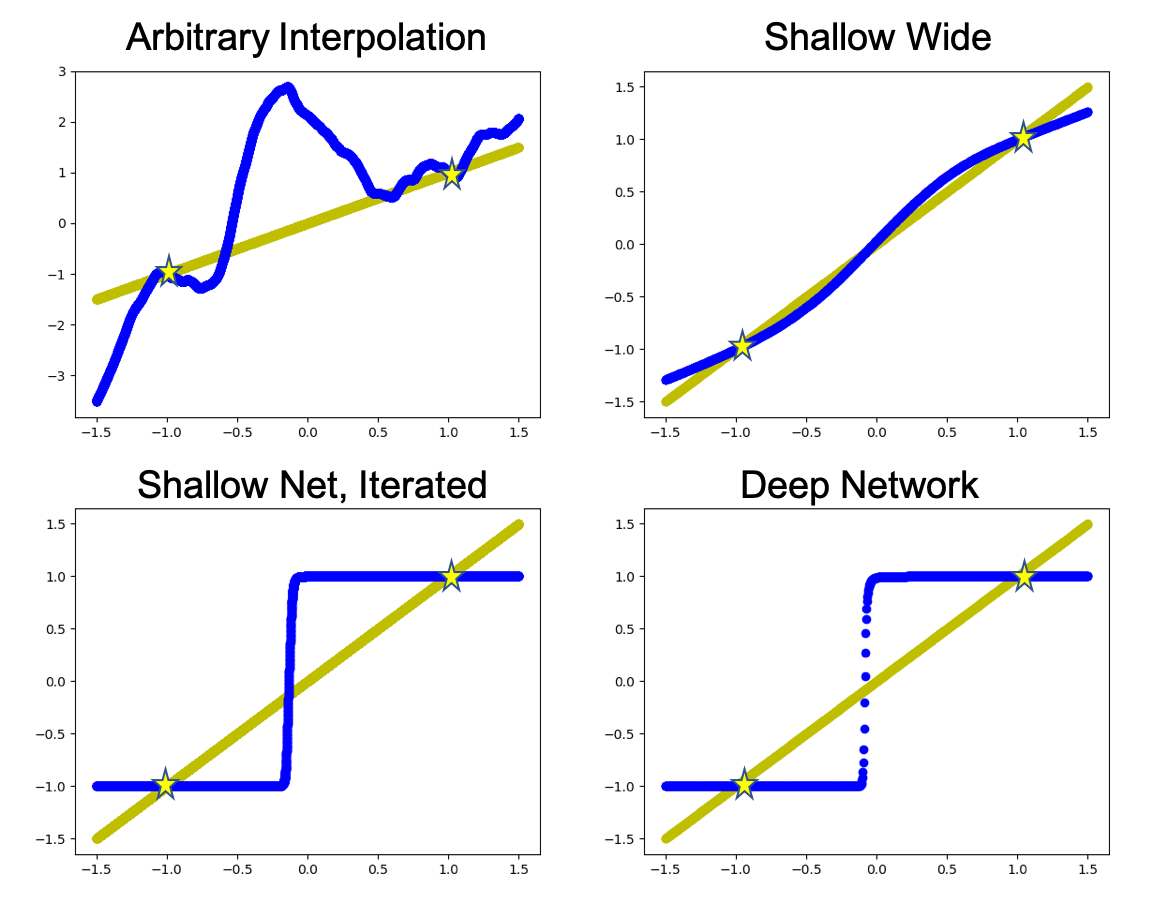}
        \caption{}
        \label{fig:SchematicAutoencoders}
 \end{subfigure}
 \hspace{0.5cm}
 \begin{subfigure}[t]{0.4\textwidth}
        \centering
        \includegraphics[scale=0.4]{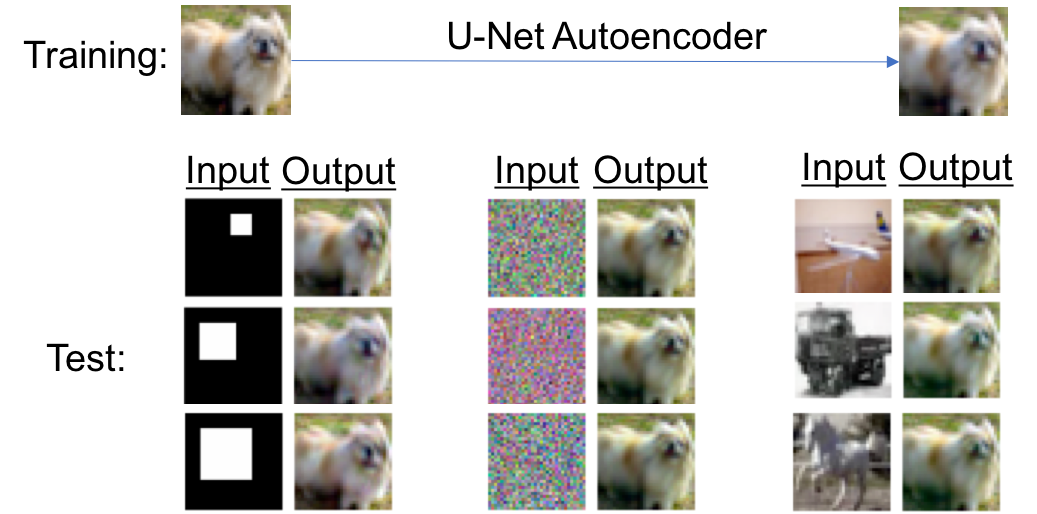}
        \caption{}
 \end{subfigure}
   \caption{(a) (Top left) Interpolating solution of an overparameterized autoencoder (blue line) that perfectly fits the training data (stars). Such an arbitrary solution is possible, but the inductive bias of autoencoders results in the following solutions instead. (Top right) The solution learned by a shallow overparameterized autoencoder is contractive towards the training examples, which is illustrated by iteration (bottom left). (Bottom right) The solution learned by a deep overparameterized autoencoder is strongly attractive towards training examples. (b) U-Net Autoencoder trained on a single image from CIFAR10 for 2000 iterations.  When fed random sized white squares, standard Gaussian noise, or new images from CIFAR10 the function contracts to the training image. }
       \label{fig:SchematicAutoencoders}
\end{figure}

Consistent with our findings, \cite{IdentityCrisis} studied memorization in autoencoders trained on a single example, and \cite{CloserLookAtMemorization} showed that deep neural networks are biased towards learning simpler patterns as exhibited by a qualitative difference in learning real versus random data.  Similarly, \cite{ImplicitSelfRegularization} showed that weight matrices of fully-connected neural networks exhibit signs of regularization in the training process.  

\section{Memorization in Single Layer Fully Connected Autoencoders}
\label{sec:SingleFCMemorization}

\textbf{Linear setting.} As a starting point, we consider the inductive bias of linear single layer fully connected autoencoders. This autoencoding problem can be reduced to linear regression (see Supplementary Material~A), and it is well-known that solving overparametrized linear regression by gradient descent initialized at zero converges to the minimum norm solution (see, e.g., Theorem 6.1 in \cite{LinearRegressionMinimumNorm}).  The minimum norm solution for the autoencoding problem corresponds to projection onto the span of the training data. Hence after training, linear single layer fully connected autoencoders map any input to points in the span of the training set, i.e., they memorize the training images. 

\textbf{ Nonlinear Setting.}  We now prove that this property of projection onto the span of the training data extends to nonlinear single layer fully connected autoencoders.  %
After training, such autoencoders satisfy $\phi(A x^{(i)})= x^{(i)}$ for $1 \leq i \leq n$, where $A$ is the weight matrix and $\phi$ is a given non-linear activation function (e.g.~sigmoid) that acts element-wise with $x_{j}^{(i)} \in \text{range}(\phi)$, where $x_{j}^{(i)}$ denotes the $j^{th}$ element of $x^{(i)}$ for $1 \leq i \leq n$ and $1 \leq j \leq d$. In the following, we provide a closed form solution for the matrix $A$ when initialized at $A^{(0)} = \mathbf{0}$ and computed using gradient descent on the mean squared error loss, i.e.
\begin{equation}
\label{eq_ext}
    \min_{A\in\mathbb{R}^{d\times d}}\; \frac{1}{2}\sum\nolimits_{i=1}^{n}(x^{(i)} - \phi(Ax^{(i)}))^T(x^{(i)} - \phi(Ax^{(i)})).
\end{equation}
Let $\phi^{-1}(y)$ be the pre-image of $y\in\mathbb{R}$ of minimum $\ell_2$ norm and for each $1\leq j\leq d$ let 
\begin{align*}
    x_{j}' = \argmax_{1 \leq i  \leq n} |\phi^{-1}(x_j^{(i)})|. 
\end{align*}
We will show that $A$ can be derived in closed form in the nonlinear overparameterized setting under the following three mild assumptions that are often satisfied in practice.

\begin{assumption}
\label{ass}
For all $j \in \{1,2, \ldots, d\}$ it holds that

\vspace{-0.1cm}
\quad\!\! (a) \quad $0 < x_{j}^{(i)} < 1$ \;for all $1 \leq i \leq n$;

\vspace{-0.1cm}
\quad\!\! (b) \quad $x_{j}^{(i)} < \phi(0)$ (or $x_{j}^{(i)} > \phi(0)$) \;for all $1 \leq i \leq n$;

\quad\!\! (c) \quad $\phi$ satisfies one of the following conditions:
\begin{enumerate}
    \item[(1)] if $\phi^{-1}(x_{j}') > 0$ then $\phi$ is strictly convex \& monotonically decreasing on $[0, \phi^{-1}(x_{j}')]$
    \item[(2)] if $\phi^{-1}(x_{j}') > 0$, then $\phi$ is strictly concave \& monotonically increasing on $[0, \phi^{-1}(x_{j}')]$
    \item[(3)] if $\phi^{-1}(x_{j}') < 0$, then $\phi$ is strictly convex \& monotonically increasing on $[\phi^{-1}(x_{j}'), 0]$
    \item[(4)] if $\phi^{-1}(x_{j}') < 0$, then $\phi$ is strictly concave \& monotonically decreasing on $[\phi^{-1}(x_{j}'), 0]$
\end{enumerate}
\end{assumption}
Assumption (a) typically holds for un-normalized images. Assumption (b) is satisfied for example when using a min-max scaling of the images. Assumption (c) holds for many nonlinearities used in practice including the sigmoid and tanh functions. 

To show memorization in overparametrized  nonlinear single
layer fully connected autoencoders, we first show how to reduce the non-linear setting to the linear setting.   

\begin{theorem} 
\label{thm:NonlinearFCAutoencoderMemorization}
Let $n<d$ (overparametrized setting). Under Assumption~\ref{ass}, solving (\ref{eq_ext}) to achieve  $\phi(A x^{(i)}) \approx x^{(i)}$ using  a variant of gradient descent (with an adaptive learning rate as described in Supplementary Material~B) initialized at $A^{(0)} = \mathbf{0}$ converges to a solution $A^{(\infty)}$ that satisfies the linear system $A^{(\infty)}x^{(i)} = \phi^{-1}(x^{(i)})$ for all $1 \leq i \leq n$.  
 \end{theorem}

The proof is presented in Supplementary Material~B.  Given our empirical observations using a constant learning rate, we suspect that the adaptive learning rate used for gradient descent in the proof is not necessary for the result to hold. 

As a consequence of Theorem \ref{thm:NonlinearFCAutoencoderMemorization}, the single layer nonlinear autoencoding problem can be reduced to a linear regression problem.  
This allows us to define a memorization property for nonlinear systems by introducing nonlinear analogs of an eigenvector and the span. 


\begin{definition}[$\phi$-eigenvector]
\label{def:PhiEigenvector}
Given a matrix $A \in \mathbb{R}^{d \times d}$ and element-wise nonlinearity $\phi$, a vector $u \in \mathbb{R}^d$ is a $\phi$-eigenvector of $A$ with $\phi$-eigenvalue $\lambda$ if $\phi(Au) = \lambda u$.
\end{definition}

\begin{definition}[$\phi$-span]
\label{def:Nonlinear Span}
Given a set of vectors $U = \{u_1, \ldots u_r\}$ with $u_i \in \mathbb{R}^{d}$  and an element-wise nonlinearity $\phi$, let $\phi^{-1}(U) = \{\phi^{-1}(u_1) \ldots \phi^{-1}(u_r)\}$. The nonlinear span of $U$ corresponding to $\phi$ (denoted $\phi$-$\spn(U)$) consists of all vectors $\phi(v)$ such that $v \in \spn(\phi^{-1}(U))$. 
\end{definition}

The following corollary characterizes memorization for nonlinear single layer fully connected autoencoders.  

\begin{corollary}[Memorization in non-linear single layer fully connected autoencoders]
Let $n<d$ (overparametrized setting) and let $A^{(\infty)}$ be the solution to (\ref{eq_ext}) using a variant of gradient descent with an adaptive learning rate initialized at $A^{(0)} = \mathbf{0}$. Then under Assumption~\ref{ass}, $\rank(A^{(\infty)})=\dim (\spn(X))$; in addition, the training examples $x^{(i)}$, $1 \leq i \leq n$, are $\phi$-eigenvectors of $A^{(\infty)}$ with eigenvalue $1$  and $\phi(A^{(\infty)}y) \in \phi$-$\spn(X)$ for any $y \in \mathbb{R}^{d}$. 
\end{corollary}

\begin{proof}
Let $S$ denote the covariance matrix of the training examples and let $r:=\textrm{rank}(S)$. It then follows from Theorem \ref{thm:NonlinearFCAutoencoderMemorization} and the minimum norm solution of linear regression that $\textrm{rank}(A^{(\infty)})\leq r$. Since in the overparameterized setting, $A^{(\infty)}$ achieves $0$ training error, the training examples satisfy $\phi(A^{(\infty)}x^{(i)}) = x^{(i)}$ for all $1 \leq i \leq n$, which implies that the examples are $\phi$-eigenvectors with eigenvalue $1$.  Hence, it follows that $\textrm{rank}(A^{(\infty)})\geq r$ and thus $\textrm{rank}(A^{(\infty)})=r$.  Lastly, since the $\phi$-eigenvectors are the training examples, it follows that $\phi(A^{(\infty)}y) \in \phi$-$\spn(X)$ for any $y \in \mathbb{R}^d$.
\end{proof}

In the linear setting, memorization is given by the trained network projecting inputs onto the span of the training data; our result generalizes this notion of memorization to the nonlinear setting.

\begin{figure*}[!b]
\centering
\includegraphics[scale=0.3]{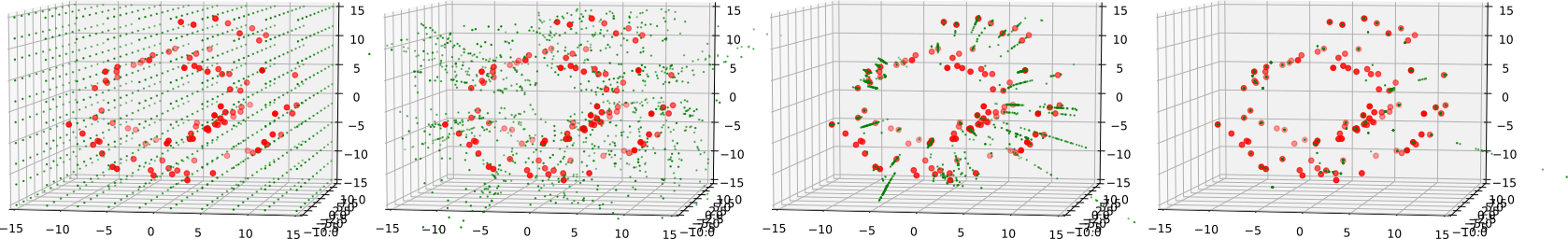}
\caption{A fully-connected autoencoder (10 layers, 512 hidden neurons) was trained to convergence on 100 samples (red points) from the 3D Swiss Roll dataset~\cite{scikit-learn}. The trajectories of 1000 uniformly spaced points (green points) were computed by iterating the autoencoder over the points. From left to right, the plots represent the results after 0, 11, 111, and 1111 iterations. Most of the trajectories converge to training examples, thereby demonstrating that the autoencoder is contractive towards training examples.}
\label{fig:SwissRoll}
\end{figure*}


\section{Memorization in Deep Autoencoders through Contractive Maps}
\label{sec:ContractiveMaps}

While single layer fully connected autoencoders memorize by learning solutions that produce outputs in the nonlinear span of the training data, we now demonstrate that deep autoencoders exhibit a stronger form of inductive bias by learning maps that are \emph{locally contractive} at training examples. To this end, we analyze autoencoders within the framework of discrete dynamical systems.

\subsection{Preliminaries: Discrete Dynamical Systems}
 A discrete dynamical system over a space $\mathcal{X}$ is defined by the relation: $x_{t+1} = f(x_{t})$ where $x_{t} \in \mathcal{X}$ is the state at time $t$ and $f: \mathcal{X} \rightarrow \mathcal{X}$ is a map describing the evolution of the state. Given an initial state $x_0 \in \mathcal{X}$, the trajectories of a discrete dynamical system can be computed by iterating the map~$f$, i.e.~$x_{t} = f^t(x_0)$. 

Fixed points occur where $x = f(x)$ and fall under two main characterizations: attractors and repellers. A fixed point $x$ is called an \emph{attractor} or \emph{stable fixed point} if trajectories that are sufficiently close converge to $x$. Conversely, $x$ is called a \emph{repeller} or \emph{unstable fixed point} if trajectories near $x$ move away. 
The following theorem provides sufficient first-order criteria for determining whether a fixed point is stable or unstable:

\begin{theorem} 
\label{thm:Attractor Characterization}
A point $x \in X$ is an attractor if the largest eigenvalue of the Jacobian, $\mathbf{J}$, of $f$ is strictly less than $1$. Conversely, $x$ is repeller if the largest eigenvalue of $\mathbf{J}$ is strictly greater than $1$.
\end{theorem}

Intuitively, the constraint on the Jacobian means that attractors are fixed points at which the function $f$ is ``flatter".  As an example, consider the function shown in the top-right of Figure \ref{fig:SchematicAutoencoders}. Note that the derivative of the function is less than $1$ at the training examples. As a result, Theorem \ref{thm:Attractor Characterization} implies that the training examples are attractors: iterating the map over almost all points will result in convergence to one of the two training examples (Figure \ref{fig:SchematicAutoencoders}, bottom left).

\subsection{Deep Autoencoders are Locally Contractive at Training Examples}
We consider discrete dynamical systems in which the map $f$ is given by a deep fully-connected autoencoder. Since an autoencoder is trained to satisfy $f(x) = x$ for all training examples, it is clear that the training examples are fixed points of this system. We now show that when $f$ is an overparameterized autoencoder, the training examples are not only fixed points, but also attractors. To do this, we experimentally simulate trajectories from discrete dynamical systems in which the map $f$ is given by deep fully-connected autoencoder trained on the Swiss Roll dataset~\cite{scikit-learn} (see Figure~\ref{fig:SwissRoll}) and the MNIST dataset~\cite{mnist-lecun1998} (see Figure~\ref{fig:MNIST-contraction}). 

\begin{figure*}[!t]
\centering
\begin{subfigure}[t]{0.45\textwidth}
\centering
\includegraphics[scale=0.5]{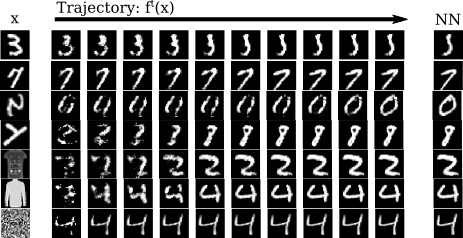}
\caption{}
\end{subfigure}
\begin{subfigure}[t]{0.45\textwidth}
\centering
\includegraphics[scale=0.5]{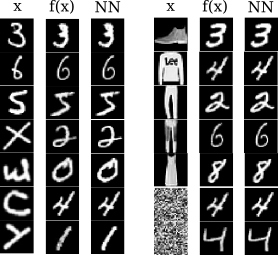}
\caption{}
\end{subfigure}
\caption{(a) A fully connected autoencoder (7 layers, 128 hidden neurons) was trained to convergence on 100 samples from the MNIST dataset. We show the trajectories of random test examples obtained by iterating the autoencoder over the images together with the nearest neighbor (NN) training image. (b) A fully connected autoencoder (14 layers, 128 hidden neurons) was trained to convergence on 20 samples from the MNIST dataset. Here the training examples are \emph{superattractors}: the autoencoder maps arbitrary input images directly to a training example.}
\label{fig:MNIST-contraction}
\end{figure*}

For the Swiss Roll dataset, we follow the trajectories of a grid of 1000 initial points. As shown in Figure~\ref{fig:SwissRoll}, trajectories from various points on the grid (shown as green points) converge to the training points (shown as red points). For the MNIST dataset, we trace the trajectories using test examples or other arbitrary images as initial points. As shown in Figure~\ref{fig:MNIST-contraction}a, the trajectories converge almost exclusively to the training examples. 

The observation that the training examples are attractors of the system implies, by Theorem \ref{thm:Attractor Characterization}, that the autoencoder $f$ is flatter or \emph{locally contractive} at the training examples. This demonstrates a form of inductive bias that we refer to as memorization: while the overparameterized autoencoders used in our experiments have the capacity to learn more complicated functions that interpolate between the training examples, gradient descent converges to a simpler solution that is contractive at the training examples. 

\begin{figure*}[!b]
\centering
\begin{subfigure}[t]{0.45\textwidth}
\includegraphics[scale=0.4]{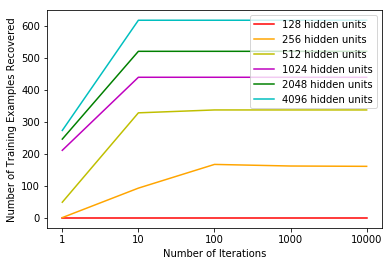}
\caption{Autoencoders of varying width.}
\end{subfigure}
\begin{subfigure}[t]{0.45\textwidth}
\includegraphics[scale=0.4]{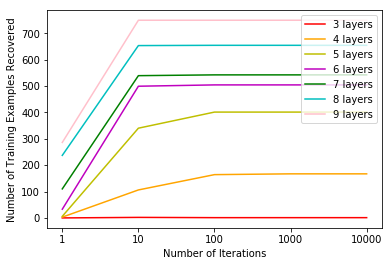}
\caption{Autoencoders of varying depth.}
\end{subfigure}
\caption{Autoencoders of varying width and depth were trained on 1000 samples from the MNIST dataset. X-axis: number of iterations. Y-axis: recovery probability of training examples.}
\label{fig:MNIST-wide}
\end{figure*}

\subsection{Contraction Depends on Network Architecture}

Next, we investigate the effect that width and depth of the autoencoder have on the contraction of the autoencoder towards the training examples. To quantify the extent to which the contraction around training examples occurs for different autoencoders, we propose to measure the \emph{recovery probability} $R_t$ of training examples as a function of number of iterations $t$. Specifically, we compute,
$$ R_t := \mathbb{P}_{x_0 \sim P_{\textrm{train}}}\{ \min_{\tilde{x} \in D_{\textrm{test}}}||f^t(\tilde{x}) - x_0|| < \epsilon\}, $$
where $P_{\textrm{train}}$ is the empirical distribution of training examples and $D_{\textrm{test}}$ is a set of initial states. This metric reflects both the basin of attraction of the training examples and their rate of convergence and is plotted for various MNIST autoencoders in Figure \ref{fig:MNIST-wide}.

Interestingly, we found that increasing width and depth both increase the recovery rate (Figure \ref{fig:MNIST-wide}). This is surprising because increasing the capacity of the autoencoder should increase the class of functions that can be learned, which should in turn result in more arbitrary interpolating solutions between the training examples. Yet the observation that the solutions are now \emph{more contractive} towards the training examples suggests that self-regularization is at play and increases with the parameterization of the network. In fact, for a sufficiently large network, the training examples become \emph{superattractors}, or attractors that contract neighboring points at a faster than geometric rate. This is demonstrated in Figure \ref{fig:MNIST-contraction}b, which shows that arbitrary inputs are mapped directly to a training image in a single iteration.

To understand why the contractive property of a deep autoencoder depends on width, we provide the following theoretical insight: for a 2-layer neural network with ReLU activations and a sufficiently large number of hidden neurons, training with gradient descent results in a function that is contractive towards the training data. 

\begin{theorem} Consider a 2-layer autoencoder $A$ represented by $f(x) = W_2\phi(W_{1} x + b)$ where $x \in \mathbb{R}^{d}, W_1 \in \mathbb{R}^{k \times d}, W_2 \in \mathbb{R}^{d \times k}, b \in \mathbb{R}^{k}$ and $\phi$ is the ReLU function.  If $A$ is trained on $X = \{x^{(1)} \ldots x^{(n)}\}$ and assuming, 
\begin{enumerate}
    \item[(a)] the weights $W_1$ are fixed,
    \item[(b)] $\phi(W_1x^{(i)} + b)$ are almost surely orthogonal due to nonlinearity for all $i \in \{1 \ldots n\}$,
    \item[(c)] $W_2^{(0)} = \mathbf{0}$, coordinates of $W_1, b $ are i.i.d. with zero mean and finite second moment.
\end{enumerate}
then as $k \rightarrow \infty$, all training examples are stable fixed points of $A$.  
\end{theorem}

The proof is presented in Supplementary Material~C.  Assumption (a) of fixing the weights of a layer is a technique that has also been used before  to study the convergence of overparameterized neural networks using gradient descent~\cite{brutzkus2017sgd,li2018learning}. These works showed that the inductive bias of neural networks optimized by gradient descent results in good generalization in the classification setting. We show here that training autoencoders by gradient descent results in an inductive bias towards functions that are contractive towards the training examples.





\begin{figure}[!b]
\centering
   \begin{subfigure}[t]{0.4\textwidth}
        \centering
        \includegraphics[scale=0.3]{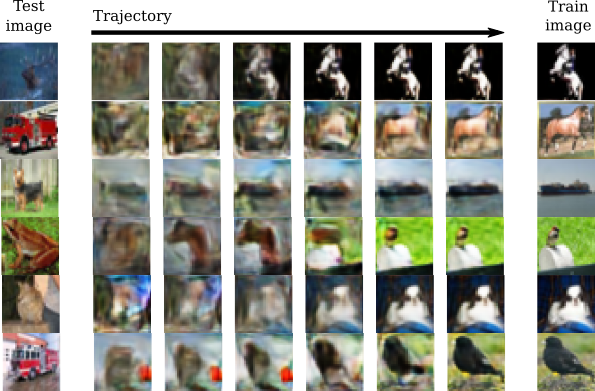}
        \caption{Training Examples are Attractors.}
        \label{fig:cifar_iterated}
 \end{subfigure}
 \hspace{0.5cm}
 \begin{subfigure}[t]{0.4\textwidth}
        \centering
        \includegraphics[scale=0.18]{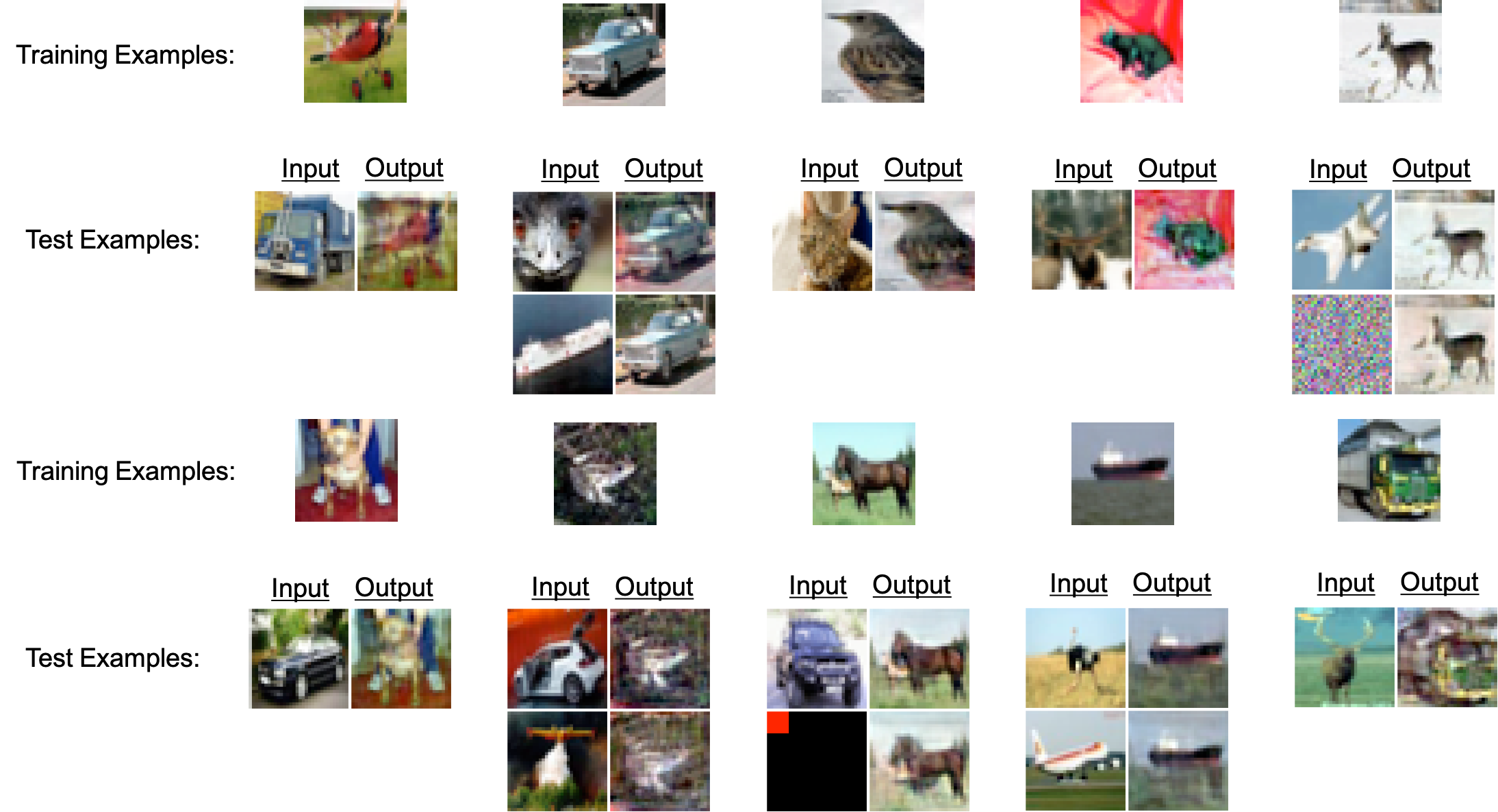}
        \caption{Training Examples are Superattractors.}
        \label{fig:cifar_strong_attractors}
 \end{subfigure}
\caption{(a) A convolutional autoencoder was trained to convergence on 100 samples from the CIFAR dataset. We show the trajectories of random test examples obtained by iterating the autoencoder over the images together with the nearest neighbor (NN) training image. (b) When trained on 10 examples, one from each class of CIFAR10, the nonlinear autoencoder maps arbitrary input images to training examples. The downsampling architecture used is from Figure 1a in Supplementary Material~G modified with leaky ReLU activations.}
\label{fig:ConvolutionIteration}
\end{figure}

\section{Memorization in Convolutional Autoencoders}
\label{sec:ConvolutionalAutoencoders}

Having characterized the inductive bias of fully connected autoencoders in the previous section, we now demonstrate that memorization is also present in convolutional autoencoders.  As an example, Figure~\ref{fig:ConvolutionIteration} shows two deep convolutional autoencoders that learn maps which are contractive to training examples.  The first autoencoder (Figure \ref{fig:cifar_iterated}) is trained on 100 images of CIFAR10 \cite{CIFAR10} and iterating the map yields individual training examples.  The second autoencoder (Figure \ref{fig:cifar_strong_attractors}) is trained on 10 images from CIFAR10, and after just 1 iteration, individual training examples are output (i.e. the training examples are superattractors).   

In the following we show that in contrast to fully connected autoencoders, depth is required for memorization in convolutional autoencoders. This is because the weight matrices in convolutional networks are subject to additional sparsity constraints as compared to fully connected networks (See Supplementary Material~E). The following theorem states that overparameterized shallow linear convolutional autoencoders learn a full-rank solution and thus do not memorize.

\begin{theorem}
\label{thm:SingleConvolution}
A single filter convolutional autoencoder with kernel size $k$ and $\frac{k-1}{2}$ zero padding trained on an image $x \in \mathbb{R}^{s \times s}$ using gradient descent on the mean squared error loss learns a rank $s^2$ solution.  
\end{theorem}

The proof is presented in Supplementary Material ~D.  Next we provide a lower bound on the number of layers required for memorization in a linear convolutional autoencoder. The proof requires the following lemma, which states that a linear autoencoder (i.e. matrix) with just one forced zero cannot memorize arbitrary inputs.

\begin{lemma}
\label{lem:SingleForcedZero}
A single layer linear autoencoder, represented by a matrix $A \in \mathbb{R}^{d \times d}$ with a single forced zero entry cannot memorize arbitrary $v \in \mathbb{R}^d$.
\end{lemma}

The proof follows directly from the fact that in the linear setting, memorization corresponds to projection onto the training example and thus cannot have a zero in a fixed data-independent entry.  Since convolutional layers result in a sparse weight matrix with forced zeros, linear convolutional autoencoders with insufficient layers to eliminate zeros through matrix multiplications cannot memorize arbitrary inputs, regardless of the number of filters per layer. This is the key message of the following theorem.

\begin{theorem}
\label{thm:DepthRequired}
At least $s-1$ layers are required for memorization (regardless of the number of filters per layer) in a linear convolutional autoencoder with filters of kernel size $3$ applied to $s\times s$ images.  
\end{theorem}

Multiplication of $s-1$ such layers eliminates sparsity in the  resulting operator, which is a requirement for memorization due to Lemma \ref{lem:SingleForcedZero}.  Importantly, Theorem~\ref{thm:DepthRequired} shows that adding filters cannot make up for missing depth, i.e., overparameterization through depth rather than filters is necessary for memorization in convolutional autoencoders. The following corollary emphasizes this point.

\begin{corollary}
A 2-layer linear convolutional autoencoder with  filters of kernel size $3$ and stride $1$ for the hidden representation cannot memorize images of size $4 \times 4$ or larger, independently of the number of filters.  
\end{corollary}

In Supplementary Material~F, we provide further empirical evidence that depth is also sufficient for memorization (see also Figure~\ref{fig:ConvolutionIteration}), and refine the lower bound from Theorem \ref{thm:DepthRequired} to a lower bound of $\lceil \frac{s^4}{9} \rceil$ layers needed to identify memorization in linear convolutional autoencoders.  While the number of layers needed for memorization are large according to this lower bound, in Supplementary Material~G, we show empirically that \emph{downsampling} through strided convolution allows a network to memorize with far fewer layers.  Finally, in Supplementary Material~H, we also provide evidence of how these bounds obtained in the linear setting apply to the nonlinear setting.  

\section{Robustness of Memorization}
\label{sec:RobustnessOfMemorization}

\par{\textbf{Memorization and Overfitting.}} A natural question is whether the memorization of training examples observed in the previous sections is synonymous with overfitting of the training data. In fact, we found that it is possible for a deep autoencoder to contract towards training examples even as it faithfully learns the identity function over the data distribution. The following theorem states that all training points can be memorized by a neural network while achieving an arbitrarily small expected reconstruction error.

\begin{theorem}
For any training set $\{x_i\}_{i \in [1..n]}$ and for any $\epsilon > 0$, there exists a 2-layer fully-connected autoencoder $f$ with ReLU activations and $(n+1) \cdot d$ hidden units such that (1) the expected reconstruction error loss of $f$ is less than $\epsilon$ and (2) the training examples $\{x_i\}_{i \in [1..n]}$ are attractors of the discrete dynamical system with respect to $f$.
\end{theorem}

Empirical evidence of autoencoders that simultaneously exhibit memorization and achieve near-zero reconstruction error can be found in Supplementary Material~I. Additional support for the claim that memorization is not synonymous with overfitting lies in the observation that memorization occurs throughout the training process and is present even with early stopping (see Figure \ref{fig:EarlyStopping}). 

\begin{figure}[!t]
    \centering
        \begin{subfigure}[t]{0.45\textwidth}
        \centering
        \includegraphics[height=1in]{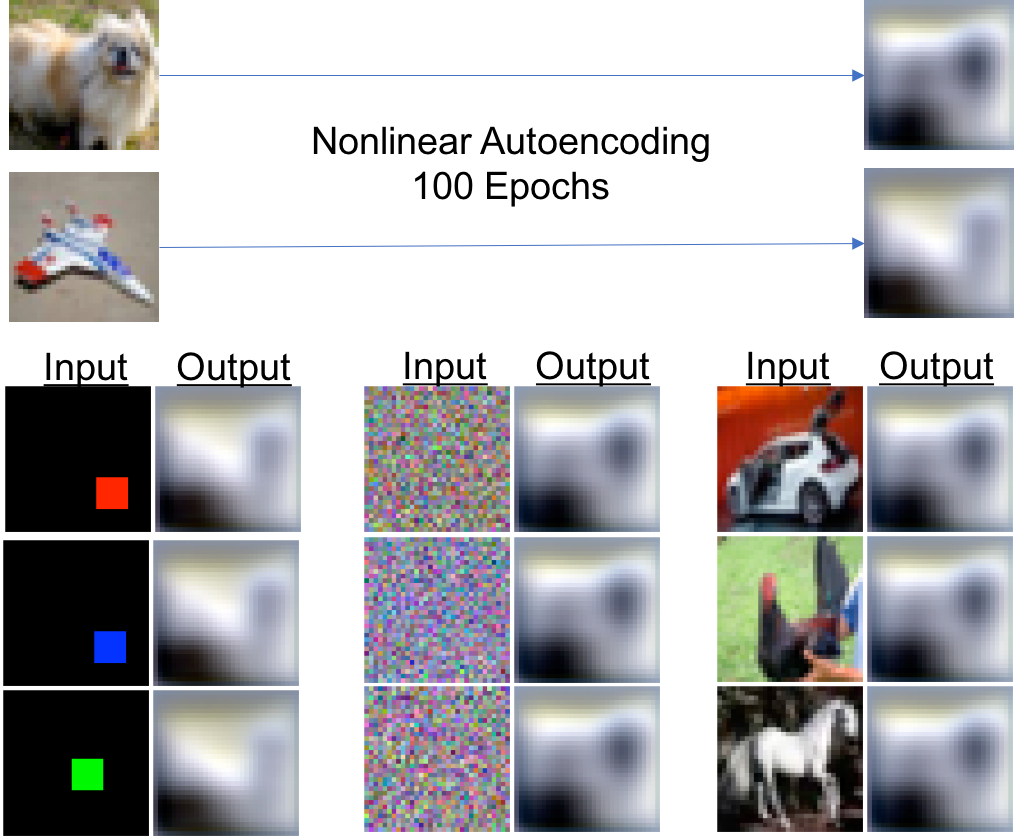}
        \caption{Nonlinear 100 training epochs.}
        \label{fig:Nonlinear100Epochs}   
    \end{subfigure}%
    ~ 
    \begin{subfigure}[t]{0.45\textwidth}
        \centering
        \includegraphics[height=1in]{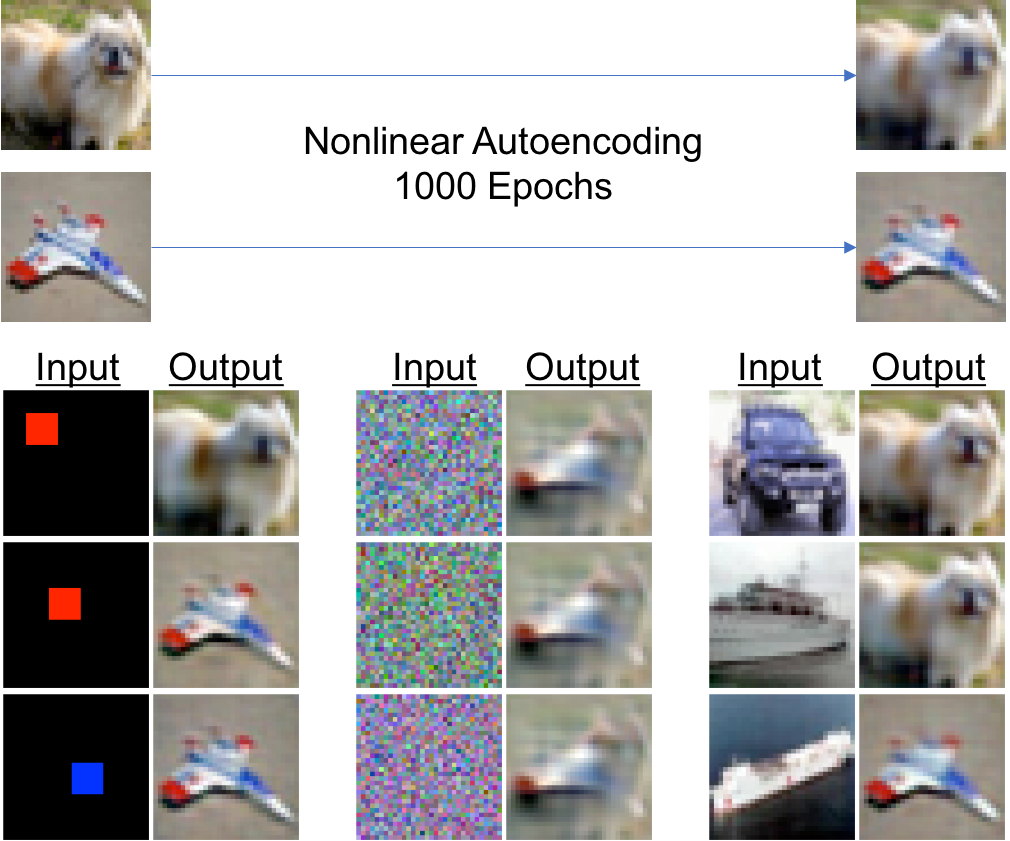}
        \caption{Nonlinear 1000 training epochs.}
        \label{fig:Nonlinear1000Epochs}
    \end{subfigure}%
\caption{The network from Figure 1a in Supplementary Material~G (modified with leaky relu nonlinearities) memorizes learned representations of training data as superattractors throughout the training process regardless of early stopping.}
    \label{fig:EarlyStopping}   
\end{figure}

\par{\textbf{Effect of Initialization.}} In Section~\ref{thm:NonlinearFCAutoencoderMemorization}, we identified memorization in fully connected autoencoders initialized at zero.  We now consider the impact of nonzero initialization.  For the linear setting, the initialization vector has a component that is orthogonal to the span of the training data.  That component is preserved throughout the training process as update vectors are contained in the span of the training data.

If the initialization is small (low variance, zero mean), then during training the initialization is subsumed by the component in the direction of the span of the training data.  However, if the initialization is large, then the solution is dominated by the random initialization vector.  Since memorization corresponds to learning a projection onto the span of the training data, using a large initialization vector leads to  noisy memorization.  This  demonstrates that while untrained autoencoders can be analyzed  using tools from random matrix theory \cite{RandomWeightAutoencoders}, the properties of the solution before and after training can be very different. In Supplementary Material~I, we provide further details regarding such a perturbation analysis in the linear setting and examples to illustrate the effect of different  initializations used in practice.

\section{Conclusions and Future Work}
\label{sec:final_sec}

In this work, we showed that overparameterized autoencoders learn maps that concentrate around the training examples, an inductive bias we defined as memorization.  While it is well-known that linear regression converges to a minimum norm solution when initialized at zero, we tied this phenomenon to memorization in nonlinear single layer fully connected autoencoders, showing that they produce output in the nonlinear span of the training examples.  We then demonstrated that nonlinear fully connected autoencoders learn maps that are contractive around the training examples, and so iterating the maps outputs individual training examples.  Lastly, we showed that with sufficient depth, convolutional autoencoders exhibit similar memorization properties.  

Interestingly, we observed that the phenomenon of memorization is pronounced in deep nonlinear autoencoders, where nearly arbitrary input images are mapped to training examples after one iteration (i.e., the training examples are superattractors).  This phenomenon is  similar to that of FastICA in Independent Component Analysis~\cite{FastICA} or more general nonlinear eigenproblems~\cite{EigsDecompFunctions}, where every  ``eigenvector" (corresponding to training examples in our setting) of certain iterative maps has its own basin of attraction.  In particular, increasing depth plays the role of increasing the number of iterations in those methods.  

The use of deep networks with near zero initialization is the current standard for image classification tasks, and we expect that our memorization results carry over to these application domains. We note that memorization is a particular form of interpolation (zero training loss) and interpolation has been demonstrated to be capable of generalizing to test data in neural networks and a range of other methods~\cite{DoubleDescent,RethinkingGeneralization}.  Our work could thus provide a mechanism to link overparameterization and memorization with generalization properties observed in deep networks.
 
\section*{Acknowledgements}
The authors thank the Simons Institute at UC Berkeley for hosting them during the program on ``Foundations of Deep Learning'', which facilitated this work. A.~Radhakrishnan, K.D.~Yang and C.~Uhler were partially supported by the National Science Foundation (DMS-1651995), Office of Naval Research (N00014-17-1-2147 and N00014-18-1-2765), IBM, and a Sloan Fellowship to C.~Uhler. K.~D.~Yang was also supported by an NSF Graduate Fellowship. M.~Belkin acknowledges support from NSF (IIS-1815697 and IIS-1631460). The Titan Xp used for this research was donated by the NVIDIA Corporation.

\bibliography{example_paper}
\bibliographystyle{plain}

\appendix

\section{Minimum Norm Solution for Linear Fully Connected Autoencoders}
\label{sec:LinearRegression}
In the following, we analyze the solution when using gradient descent to solve the autoencoding problem for the system $Ax^{(i)} = x^{(i)}$ for $1 \leq i \leq n$ with $x^{(i)} \in \mathbb{R}^d$.  The loss function is
\begin{equation*}
    L = \frac{1}{2}\displaystyle\sum\limits_{i=1}^{n}(Ax^{(i)} - x^{(i)})^T(Ax^{(i)} - x^{(i)})
\end{equation*}
and the gradient with respect to the parameters $A$ is
\begin{equation*}
    \frac{\partial L}{\partial A} = (A - I)\displaystyle\sum\limits_{i=1}^{n}(x^{(i)}x^{(i)^T}).
\end{equation*}
Let $S = \displaystyle\sum\limits_{i=1}^{n}(x^{(i)}x^{(i)^T})$.  Hence gradient descent with learning rate $\gamma > 0$ will proceed according to the equation:
\begin{equation*}
\begin{split}
    A^{(t+1)} &= A^{(t)} + \gamma (I - A^{(t)})S \\
    &= A^{(t)}(I - \gamma S) + \gamma S \\
\end{split}
\end{equation*}
Now suppose that $A^{(0)} = \mathbf{0}$, then we can directly solve the recurrence relation for $t > 0$, namely
\begin{equation*}
    A^{(t)} = I - (I - \gamma S)^t
\end{equation*}
Note that $S$ is a real symmetric matrix, and so it has eigen-decomposition $S = Q\Lambda Q^T$ where $\Lambda$ is a diagonal matrix with eigenvalue entries $\lambda_1 \geq \lambda_2 \geq \ldots \geq \lambda_r$ (where $r$ is the rank of $S$).  Then:
\begin{align*}
    A^{(t)} &= I - Q(I - \gamma \Lambda)^tQ^T \\
    &= Q (I - (I - \gamma \Lambda)^t)Q^T.
\end{align*}
Now if $\gamma < \frac{1}{\lambda_1}$, then we have that:
\begin{equation*}
\begin{split}
    A^{(\infty)} = Q\begin{bmatrix}
            I_{r \times r} & \mathbf{0}_{r \times d-r}  \\
            \mathbf{0}_{d-r \times r} & \mathbf{0}_{d - r \times d-r} \\
    \end{bmatrix}Q^T,
\end{split} 
\end{equation*}
which is the minimum norm solution.  

\section{Proof for Nonlinear Fully Connected Autoencoder}
In the following, we present the proof of Theorem 2 from the main text.  

\begin{proof}
As we are using a fully connected network, the rows of the matrix $A$ can be optimized independently during gradient descent.  Thus without loss of generality, we only consider the convergence of the first row of the matrix $A$ denoted $A_1 = [a_1, a_2, \ldots a_d]$ to find $A_1^{(\infty)}$.
The loss function for optimizing row $A_1$ is given by:
\begin{equation*}
    L = \frac{1}{2}\displaystyle\sum\limits_{i=1}^{n}(x_1^{(i)} - \phi(A_1x^{(i)}))^2.
\end{equation*}
Our proof involves using gradient descent on $L$ but with a different adaptive learning rate per example.  That is, let $\gamma_i^{(t)}$ be the learning rate for training example $i$ at iteration $t$ of gradient descent.  Without loss of generality, fix $j\in\{1,\dots ,d\}$. The gradient descent equation for parameter $a_j$ is:
\begin{equation*}
    a_j^{(t + 1)} = a_j^{(t)} + \displaystyle\sum\limits_{i=1}^{n}\gamma_i^{(t)}x_j^{(i)}\phi'(A_1^{(t)}x^{(i)})(x_j^{(i)} - \phi(A_1^{(t)}x^{(i)}))
\end{equation*}
To simplify the above equation, we make the following substitution
\begin{equation*}
    \gamma_i^{(t)} = -\frac{\gamma_i}{\phi'(A_1^{(t)}x^{(i)})},
\end{equation*}
i.e., the adaptive component of the learning rate is the reciprocal of $\phi'(A_1^{(t)}x^{(i)})$ (which is nonzero due to monotonicity conditions on $\phi$).  Note that we have included the negative sign so that if $\phi$ is monotonically decreasing on the region of gradient descent, then our learning rate will be positive.  Hence the gradient descent equation simplifies to
\begin{equation*}
    a_j^{(t + 1)} = a_j^{(t)} + \displaystyle\sum\limits_{i=1}^{n}\gamma_i x_j^{(i)}( \phi(A_1^{(t)}x^{(i)}) - x_j^{(i)}).    
\end{equation*}
Before continuing, we briefly outline the strategy for the remainder of the proof.  First, we will use assumption (c) and induction to upper bound the sequence $(\phi(A_1^{(t)}x^{(i)}) - x_j^{(i)})$ with a sequence along a line segment.  The iterative form of gradient descent along the line segment will have a simple closed form and so we will obtain a coordinate-wise upper bound on our sequence of interest $A_1^{(t)}$.  Next, we show that our upper bound given by iterations along the selected line segment is in fact a coordinate-wise least upper bound.  Then we show that $A_1^{(t)}$ is a coordinate-wise monotonically increasing function, meaning that it must converge to the least upper bound established prior.  

Without loss of generality assume, $\phi^{-1}(x_k') > 0$ for $1 \leq k \leq d$.  By assumption (c), we have that for $x \in [0, \phi^{-1}(x_j^{(i)})]$,
\begin{equation*}
    \phi(x) < -\frac{\phi(0) - x_j^{(i)}}{\phi^{-1}(x_j^{(i)})}x + \phi(0) ,
\end{equation*}
since the right hand side is just the line segment joining points $(0, \phi(0))$ and $(\phi^{-1}(x_j^{(i)}), x_j^{(i)})$, which must be above the function $\phi(x)$ if the function is strictly convex.  To simplify notation, we write
\begin{equation*}
    s_i = -\frac{\phi(0) - x_j^{(i)}}{\phi^{-1}(x_j^{(i)})}.
\end{equation*}

Now that we have established a linear upper bound on $\phi$, consider a sequence $B_1^{(t)} = [b_1^{(t)}, \ldots b_d^{(t)}] $ analogous to $A_1^{(t)}$ but with updates:
\begin{equation*}
    b_j^{(t + 1)} = b_j^{(t)} + \displaystyle\sum\limits_{i=1}^{n}\gamma_i x_j^{(i)}( -s_iB_1^{(t)}x^{(i)} + \phi(0) - x_j^{(i)})    
\end{equation*}

Now if we let $\gamma_i = \frac{\gamma}{s_i}$, then we have
\begin{equation*}
    b_j^{(t + 1)} = b_j^{(t)} - \displaystyle\sum\limits_{i=1}^{n} \gamma x_j^{(i)}(B_1^{(t)}x^{(i)} - \phi^{-1}(x_j^{(i)})),   
\end{equation*}
which is the gradient descent update equation with learning rate $\gamma$ for the first row of the parameters $B$ in solving $Bx^{(i)} = \phi^{-1}(x^{(i)})$ for $1 \leq i \leq n$.  Since gradient descent for a linear regression initialized at $0$ converges to the minimum norm solution (see Appendix A), we obtain that $B_1^{(t)}x^{(i)} \in [0, \phi^{-1}(x_1')]$ for all $t \geq 0$ when $B_1^{(0)} = \mathbf{0}$.    

Next, we wish to show that $B_j^{(t)}$ is a coordinate-wise upper bound for $A_1^{(t)}$.  To do this, we first select $L$ such that $\frac{x_j^{(i)}}{x_k^{(i)}} \leq L$ for $1 \leq i \leq n$ and $1 \leq j, k \leq d$ (i.e. $L \geq 1$).  

Then, we proceed by induction to show the following:
  \begin{enumerate}
      \item  $b_j^{(t)} > a_j^{(t)}$ for all $t \geq 2$ and $1 \leq j \leq d$.
      \item For $C_1^{(t)} = [c_1^{(t)}, \ldots c_d^{(t)}] = B_1^{(t)} - A_1^{(t)}$, $\frac{c_l^{(t)}}{c_j^{(t)}} \leq L$ for $1 \leq l, j \leq d$ and for all $t \geq 2$ .
  \end{enumerate}
  
To simplify notation, we follow induction for $a_1^{(t)}$ and $b_1^{(t)}$ and by symmetry our reasoning follows for $a_j^{(t)}$ and $b_j^{(t)}$ for $2 \leq j \leq d$.  
  
\textbf{Base Cases : } 
  \begin{enumerate}
      \item Trivially we have $a_1^{(1)} = b_1^{(1)} = 0$ and so $A_1^{(0)}x = B_1^{(0)}x = 0$.
      \item We have that:
      $a_1^{(1)} = b_1^{(1)} = \displaystyle\sum\limits_{i=1}^{n}\gamma_i x_1^{(i)}( \phi(0) - x_1^{(i)})$.  Hence we have $A_1^{(1)}x = B_1^{(1)}x$. 
      \item Now for $t = 2$, 
      \begin{equation*}
          \begin{split}
            a_1^{(2)} &= a_1^{(1)} + \displaystyle\sum\limits_{i=1}^{n}\gamma_i x_1^{(i)}( \phi(A_1^{(1)}x^{(i)}) - x_1^{(i)}) \\
            b_1^{(2)} &= b_1^{(t)} + \displaystyle\sum\limits_{i=1}^{n}\gamma_i x_1^{(i)}( -s_iB_1^{(1)}x^{(i)} + \phi(0) - x_1^{(i)}) \\
          \end{split}
      \end{equation*}
      However, we know that $B_1^{(1)}x^{(i)} \in [0, \phi^{-1}(x_1')]$ and since $A_1^{(1)} = B_1^{(1)}$, $A_1^{(1)}x^{(i)} \in [0, \phi^{-1}(x_1')]$.  Hence, $b_1^{(2)} > a_1^{(2)}$ since the on the interval $[0, \phi^{-1}(x_1')]$, $\phi$ is bounded above by the line segments with endpoints $(0, \phi(0))$ and $(\phi^{-1}(x_1^{(i)}), x_1^{(i)})$.  Now for the second component of induction, we have:
      \begin{equation*}
      \begin{split}
            c_j^{(2)} &= \displaystyle\sum\limits_{i=1}^{n}\gamma_i x_j^{(i)}( -s_iB_1^{(1)}x^{(i)} + \phi(0) - \phi(A_1^{(1)}x^{(i)})) \\
            c_l^{(2)} &= \displaystyle\sum\limits_{i=1}^{n}\gamma_i x_l^{(i)}( -s_iB_1^{(1)}x^{(i)} + \phi(0) - \phi(A_1^{(1)}x^{(i)}))
      \end{split}
      \end{equation*}
      To simplify the notation, let:
      \begin{equation*}
        G_i^{(1)} = -s_iB_1^{(1)}x^{(i)} + \phi(0) - \phi(A_1^{(1)}x^{(i)})             
      \end{equation*} 
      Thus, we have
      \begin{equation*}
      \begin{split}
            \frac{c_l^{(2)}}{c_j^{(2)}} &= \frac{\displaystyle\sum\limits_{i=1}^{n}\gamma_i x_l^{(i)} G_i^{(1)}}{\displaystyle\sum\limits_{i=1}^{n}\gamma_i x_j^{(i)} G_i^{(1)}} 
            = \frac{\displaystyle\sum\limits_{i=1}^{n}\gamma_i \frac{x_l^{(i)}}{x_j^{(i)}\displaystyle\prod\limits_{p \neq i} x_j^{(p)}} G_i^{(1)}}{\displaystyle\sum\limits_{i=1}^{n}\gamma_i \frac{1}{\displaystyle\prod\limits_{p \neq i} x_j^{(p)}} G_i^{(1)}} \\
            &\leq \frac{\displaystyle\sum\limits_{i=1}^{n}\gamma_i L \frac{1}{\displaystyle\prod\limits_{p \neq i} x_j^{(p)}} G_i^{(1)}}{\displaystyle\sum\limits_{i=1}^{n}\gamma_i \frac{1}{\displaystyle\prod\limits_{p \neq i} x_j^{(p)}} G_i^{(1)}} 
            = L \\
      \end{split}
      \end{equation*}
\end{enumerate} 

\textbf{Inductive Hypothesis: } We now assume that for $t = k$, $b_1^{(k)} > a_1^{(k)}$ and so $B_1^{(k)}x > A_1^{(k)}x$.  We also assume $\frac{c_i^{(k)}}{c_j^{(k)}} \leq L$.  

  \textbf{Inductive Step: } Now we consider $t = k+1$. Since $b_1^{(k)} = a_1^{(k)} + c_1^{(k)}$ for $c_1^{(k)} > 0$ and since $x_j^{(i)} \in (0, 1)$ for all $i, j$, we have $B_1^{(k)}x^{(i)} = A_1^{(k)}x^{(i)} + \sum\limits_{j=1}^dc_j^{(k)}x_j^{(i)}$. Consider now the difference between $b_1^{(k+1)}$ and $a_1^{(k+1)}$:
  \begin{align*}
    c_1^{(k+1)} &= c_1^{(k)} + \displaystyle\sum\limits_{i=0}^{n}\gamma_i x_1^{(i)} (-s_iB_1^{k}x^{(i)} + \phi(0) - \phi(A_1^{(k)}x^{(i)})) \\
    &= c_i^{(k)}  + \displaystyle\sum\limits_{i=0}^{n}\gamma_i x_1^{(i)} (-s_iA_1^{(k)}x^{(i)}+ \phi(0) - \phi(A_1^{(k)}x^{(i)})) \\
    & \hspace{10mm}- \displaystyle\sum\limits_{i=0}^{n}\gamma_i x_1^{(i)}s_i\sum\limits_{j=1}^dc_j^{(k)}x_j^{(i)}\\
    &> c_1^{(k)} - \displaystyle\sum\limits_{i=0}^{n}\gamma_i x_1^{(i)}s_i\sum\limits_{j=1}^dc_j^{(k)}x_j^{(i)} \\
    &\geq c_1^{(k)} - \displaystyle\sum\limits_{i=0}^{n} \gamma_i s_i \sum\limits_{j=1}^d c_j^{(k)} \\
    &= c_1^{(k)} - \displaystyle\sum\limits_{i=0}^{n} \gamma_i s_i c_1^{(k)} \sum\limits_{j=1}^d \frac{c_j^{(k)}}{c_1^{(k)}} \\
    &\geq c_i^{(k)} - \displaystyle\sum\limits_{i=0}^{n} \gamma_i s_i c_1^{(k)} L d,              
  \end{align*}
  where the first inequality comes from the fact that $-sA_1^{(k)}x^{(i)} + \phi(0)$ is a point on the line that upper bounds $\phi$ on the interval $[0, \phi^{-1}(x_1')]$, and the second inequality comes from the fact that each $x_j^{(i)} < 1$.  Hence, with a learning rate of
  \begin{equation*}
      \gamma_i = \frac{\gamma}{s_i} ~~ \text{with} ~~ \gamma < \frac{1}{nLd},
  \end{equation*}
  we obtain that $c_1^{(k+1)} = b_1^{(k+1)} - a_1^{(k+1)} > 0$ as desired.  Hence, the first component of the induction is complete.  To fully complete the induction we must show that $\frac{c_i^{(k+1)}}{c_j^{(k+1)}} \leq L$ for $1 \leq l, j \leq d$.  We proceed as we did in the base case:
  \begin{equation*}
  \begin{split}
        c_l^{(k+1)} &= c_l^{(k)} \\
        &~~~~~ + \displaystyle\sum\limits_{i=1}^{n}\gamma_i x_l^{(i)}( -s_iB_1^{(k)}x^{(i)} + \phi(0) - \phi(A_1^{(k)}x^{(i)})) \\
        c_j^{(k+1)} &= c_j^{(k)} \\
        &~~~~~ + \displaystyle\sum\limits_{i=1}^{n}\gamma_i x_j^{(i)}( -s_iB_1^{(k)}x^{(i)} + \phi(0) - \phi(A_1^{(k)}x^{(i)})).
      \end{split}
      \end{equation*}
      To simplify the notation, let
      \begin{equation*}
        G_i^{(k)} = -s_iB_1^{(k)}x^{(i)} + \phi(0) - \phi(A_1^{(k)}x^{(i)}),     
      \end{equation*} and thus
      \begin{equation*}
      \begin{split}
            \frac{c_l^{(k+1)}}{c_j^{(k+1)}} &= \frac{c_l^{(k)} + \displaystyle\sum\limits_{i=1}^{n}\gamma_i x_l^{(i)}G_i^{(k)}}{c_j^{(k)} +\displaystyle\sum\limits_{i=1}^{n}\gamma_i x_j^{(i)}G_i^{(k)}} \\
            &= \frac{\frac{c_l^{(k)}}{c_j^{(k)}\displaystyle\prod\limits_{p=1}^{d} x_j^{(p)}} +   \displaystyle\sum\limits_{i=1}^{n}\gamma_i \frac{x_l^{(i)}}{x_j^{(i)}c_j^{(k)}\displaystyle\prod\limits_{p \neq i} x_j^{(p)}}G_i^{(k)}}{\frac{1}{\displaystyle\prod\limits_{p=1}^{d} x_j^{(p)}} + \displaystyle\sum\limits_{i=1}^{n}\gamma_i \frac{1}{c_j^{(k)}\displaystyle\prod\limits_{p \neq i} x_j^{(p)}}G_i^{(k)}} \\
            &\leq \frac{\frac{L}{\displaystyle\prod\limits_{p=1}^{d} x_j^{(p)}} +   \displaystyle\sum\limits_{i=1}^{n}\gamma_i \frac{L}{c_j^{(k)}\displaystyle\prod\limits_{p \neq i} x_j^{(p)}}G_i^{(k)}}{\frac{1}{\displaystyle\prod\limits_{p=1}^{d} x_j^{(p)}} + \displaystyle\sum\limits_{i=1}^{n}\gamma_i \frac{1}{c_j^{(k)}\displaystyle\prod\limits_{p \neq i} x_j^{(p)}}G_i^{(k)}} \\
            &= L \\
      \end{split}
      \end{equation*}
This completes the induction argument and as a consequence we obtain $c_l^{(t)} > 0$  and $\frac{c_l^{(t)}}{c_j^{(t)}} \leq L$ for all integers $t \geq 2$ and for $1 \leq l, j \leq d$.  
 
Hence, the sequence $b_i^{(t)}$ is an upper bound for $a_i^{(t)}$ given learning rate $\gamma \leq \frac{1}{nLd}$. By symmetry between the rows of $A$, we have that, the solution given by solving the system $Bx^{(i)} = \phi^{-1}(x^{(i)})$ for $1 \leq i \leq n$ using gradient descent with constant learning rate is an entry-wise upper bound for the solution given by solving $\phi(Ax^{(i)}) = x^{(i)}$ for $1 \leq i \leq n$ using gradient descent with adaptive learning rate per training example when $A^{(0)} = B^{(0)} = \mathbf{0}$.

Now, since the entries of $B^{(t)}$ are bounded and since they are greater than the entries of $A^{(t)}$ for the given learning rate, it follows from the gradient update equation for $A$ that the sequence of entries of $A^{(t)}$ are monotonically increasing from $0$. Hence, if we show that the entries of $B^{(\infty)}$ are least upper bounds on the entries of $A^{(t)}$, then it follows that the entries of $A^{(t)}$ converge to the entries of $B^{(\infty)}$.

Suppose for the sake of contradiction that the least upper bound on the sequence $a_j^{(t)}$ (the $j^{th}$ entry of the first row of $A$) is a $b_j^{(\infty)} - \epsilon_j$ for $\epsilon = [\epsilon_1, \ldots \epsilon_d]$ with $\epsilon_j > 0$ for $1 \leq j \leq d$.  Then
\begin{equation*}
    \phi(A_1^{(\infty)}x^{(i)}) = \phi(B_1^{(\infty)}x^{(i)} - \epsilon x^{(i)})
\end{equation*}
for $1 \leq i \leq n$. Since we are in the overparameterized setting,  at convergence $A_1^{(\infty)}$ must give $0$ loss under the mean squared error loss and so $\phi(B_1^{(\infty)}x^{(i)} - \epsilon x^{(i)}) = x_1^{(i)}$. This implies that $B_1^{(\infty)}x^{(i)} - \epsilon x^{(i)}$ is a pre-image of $x_1^{(i)}$ under $\phi$.  However, we know that $B_1^{(\infty)}x^{(i)}$ must be the minimum norm pre-image of $x_1^{(i)}$ under $\phi$.  Hence we reach a contradiction to minimiality since $B_1^{(\infty)}x^{(i)} - \epsilon x^{(i)} < B_1^{(\infty)}x^{(i)}$ as $\epsilon x^{(i)} > 0$.  This completes the proof and so we conclude that $A^(t)$ converges to the solution given by autoencoding the linear system $Ax^{(i)} = \phi^{-1}x^{(i)}$ for $1 \leq i \leq n$ using gradient descent with constant learning rate.  
\end{proof}

\section{Contraction Around Training Examples for Large Width}
\label{appendix: ContractionInfiniteWidth}

As only the last layer is trainable and as $W_2^{(0)} = \mathbf{0}$, then by the minimum norm solution to linear regression (Supplementary Material~A), we have that after training:
\begin{align*}
    W_2^{(\infty)} = \displaystyle\sum\limits_{i=1}^{n} \frac{x^{(i)} \phi(W_1 x^{(i)} + b)^T} {||\phi(W_1 x^{(i)} + b) ||_2^2}
\end{align*}

Hence:
\begin{align*}
    f(x) = \displaystyle\sum\limits_{i=1}^{n} \frac{x^{(i)} \phi(W_1 x^{(i)} + b)^T\phi(W_1 x + b)} {||\phi(W_1 x^{(i)} + b) ||_2^2}
\end{align*}

Now entry $(m, l)$ of the Jacobian is given by:
\begin{align*}
    \frac{\partial f(x)_m}{\partial x_l} &= \displaystyle\sum\limits_{i=1}^{n} \frac{x^{(i)}_m \phi(W_1 x^{(i)} + b)^T (\phi'(W_1 x + b) \odot W_1(*, l)) } {||\phi(W_1 x^{(i)} + b) ||_2^2}
\end{align*}
where $\odot$ is element-wise multiplication and $W_1(*, l)$ is column $l$ of $W_1$.  Now by assumption (b) as $\phi(W_1 x^{(i)} + b) \perp \phi(W_1 x^{(j)} + b)$ due to nonlinearity, we have that evaluating this partial derivative at training example $x^{(j)}$ gives:
\begin{align*}
    \frac{\partial f(x)_m}{\partial x_l}|_{x^{(j)}} &=  \frac{x^{(j)}_m \phi(W_1 x^{(j)} + b)^T (\phi'(W_1 x^{(j)} + b) \odot W_1(*, l)) } {||\phi(W_1 x^{(j)} + b) ||_2^2}
\end{align*}
Now by the strong law of large numbers, $\phi(W_1x^{(j)} + b)^T(\phi'(W_1x^{(j)} + b) \odot W_1(*, l))$ converges to $rc x_l^{(j)}$ almost surely, where $r$ is the expected number of nonzero entries in $\phi'(W_1x^{(j)} + b)$ and $c$ is the second moment of entries in $W_1$ and $b$.  Similarly, by the strong law of large numbers, $||\phi(Wx^{(j)} + b)||^2_2$ converges to $rc ||x^{(j)}||^2_2 + rc$.  

Hence as the denominator of entry $(m, l)$ of the Jacobian is nonzero almost surely, we have that as width $k \rightarrow \infty$:
\begin{align*}
    \frac{\partial f(x)_m}{\partial x_l}|_{x^{(j)}} \rightarrow \frac{x_m^{(j)}x_l^{(j)}}{||x^{(j)}||^2_2 + 1}
\end{align*}

Now, the eigenvalues of the Jacobian with entries described above are precisely $\lambda_1 = \frac{||x^{(j)}||^2_2}{||x^{(j)}||^2_2 + 1}$, $\lambda_2, \ldots, \lambda_d = 0$.  To see this, we note that row $r$ of this matrix is just the row $[x_1^{(j)}, x_2^{(j)}, \ldots, x_d^{(j)}]$ multiplied by $\frac{x_r^{(j)}}{||x^{(j)}||^2_2 + 1}$.  Thus the matrix is rank $1$ as each row is a multiple of the first.  Hence, the largest eigenvalue must be the trace of the matrix, which is just $\lambda_1 = \frac{||x^{(j)}||^2_2}{||x^{(j)}||^2_2 + 1}$.  Hence as the largest eigenvalue is less than 1, the trained network is contractive at example $x^{(j)}$.  

\section{Single Layer Single Filter Convolutional Autoencoder}
\label{appendix: SingleLayerSingleFilter Convolution}
In the following, we present the proof for Theorem 3 from the main text.  
\begin{proof}
A single convolutional filter with kernel size $k$ and $\frac{k-1}{2}$ zero padding operating on an image of size $s \times s$ can be equivalently written as a matrix operating on a vectorized zero padded image of size $(s + k - 1)^2$.  Namely, if $C_1, C_2, \ldots C_{k^2}$ are the parameters of the convolutional filter, then the layer can be written as the matrix
\begin{equation*}
            \begin{bmatrix}
            R \\
            R_{r:1} \\
            \vdots \\
            R_{r:s - 1} \\
            R_{r:(s + k - 1)} \\
            \vdots \\
            R_{r:(2s + k - 2)} \\
            \vdots \\
            R_{r:((s + k - 1)(s - 1))}\\
            \vdots \\
            R_{r:((s + k)(s - 1)} \\
            \end{bmatrix}\\,
\end{equation*}
where
\begin{equation*}
        R = \begin{bmatrix}
        C_1 & \ldots & C_k & \mathbf{0}_{s - 1} & \ldots & C_{k^2 - k + 1} & \ldots & C_{k^2} & \mathbf{0}_{(s + k)(s-1)}
        \end{bmatrix} \\
\end{equation*}
and $R_{r:t}$ denotes a right rotation of $R$ by $t$ elements.

Now, training the convolutional layer to autoencode example $x$ using gradient descent is equivalent to training $R$ to fit $s^2$ examples using gradient descent.  Namely, $R$ must satisfy $Rx = x_1, Rx_{l:1} = x_2, \ldots Rx_{l:(s + k - 1)(s-1) + s-1} = x_{s^2}$ where $x_{l:t}^T$ denotes a left rotation of $x^T$ by $t$ elements.  As in the proof for Theorem 1, we can use the general form of the solution for linear regression using gradient descent from Supplementary Material A to conclude that the rank of the resulting solution will be $s^2$.      
\end{proof}
\section{Linearizing CNNs}
\label{appendix:Linearizing CNNS}
In this section, we present how to extract a matrix form for convolutional and nearest neighbor upsampling layers.  We first present how to construct a block of this matrix for a single filter in Algorithm \ref{alg:MatrixConvolution}.  To construct a matrix for multiple filters, one need only apply the provided algorithm to construct separate matrix blocks for each filter and then concatenate them.  
 We first provide an example of how to convert a single layer convolutional network with a single filter of kernel size $3$ into a single matrix for $3 \times 3$ images.

First suppose we have a $ 3 \times 3$ image $x$ as input, which is shown vectorized below: 
\begin{align*}
&\begin{bmatrix} 
    x_1 & x_2 & x_3 \\
    x_4 & x_5 & x_6 \\
    x_7 & x_8 & x_9 \\
\end{bmatrix}
\rightarrow \\
&\begin{bmatrix} 
    x_1 &
    x_2 &
    x_3 &
    x_4 &
    x_5 &
    x_6 &
    x_7 &
    x_8 &
    x_9 &
\end{bmatrix}^T
\end{align*}

Next, let the parameters below denote the filter of kernel size $3$ that will be used to autoencode the above example:
\[
\begin{bmatrix} 
    A_1 & A_2 & A_3 \\
    A_4 & A_5 & A_6 \\
    A_7 & A_8 & A_9 \\
\end{bmatrix}.
\]
We now present the matrix form $A$ for this convolutional filter such that $A$ multiplied with the vectorized version of $x$ will be equivalent to applying the convolutional filter above to the image $x$ (the general algorithm to perform this construction is presented in Algorithm \ref{alg:MatrixConvolution}).  

\[
\begin{bmatrix}
    A_5 & A_6 & 0 & A_8 & A_9 & 0 & 0 & 0 & 0 \\
    A_4 & A_5 & A_6 & A_7 & A_8 & A_9 & 0 & 0 & 0 \\
    0 & A_4 & A_5 & 0 & A_7 &  A_8 & 0 & 0 & 0 \\
    A_2 & A_3 & 0 & A_5 & A_6 & 0 & A_8 & A_9 & 0 \\
    A_1 & A_2 & A_3 & A_4 & A_5 & A_6 & A_7 & A_8 & A_9 \\
    0 & A_1 & A_2 & 0 & A_4 & A_5 & 0 & A_7 & A_8 \\    
    0 & 0 & 0 & A_2 & A_3 & 0 & A_5 & A_6 & 0 \\
    0 & 0 & 0 & A_1 & A_2 & A_3 & A_4 & A_5 & A_6 \\
    0 & 0 & 0 & 0 & A_1 & A_2 & 0 & A_4 & A_5 \\
\end{bmatrix}
\]

Importantly, this example demonstrates that the matrix corresponding to a convolutional layer has a fixed zero pattern.  It is this forced zero pattern we use to prove that depth is required for memorization in convolutional autoencoders.  

In downsampling autoencoders, we will also need to linearize the nearest neighbor upsampling operation.  We provide the general algorithm to do this in Algorithm \ref{alg:MatrixUpsampling}.  Here, we provide a simple example for an upsampling layer with scale factor 2 operating on a vectorized zero padded $1 \times 1$ image:

\[
\begin{bmatrix} 
    0  & 0 &  0  & 0 & 0  & 0 &  0  & 0 & 0\\
    0  & 0 &  0  & 0 & 0  & 0 &  0  & 0 & 0\\
    0  & 0 &  0  & 0 & 0  & 0 &  0  & 0 & 0\\
    0  & 0 &  0  & 0 & 0  & 0 &  0  & 0 & 0\\
    0  & 0 &  0  & 0 & 0  & 0 &  0  & 0 & 0\\
    0  & 0 &  0  & 0 & 1  & 0 &  0  & 0 & 0\\
    0  & 0 &  0  & 0 & 1  & 0 &  0  & 0 & 0\\
    0  & 0 &  0  & 0 & 0  & 0 &  0  & 0 & 0\\
    0  & 0 &  0  & 0 & 0  & 0 &  0  & 0 & 0\\
    0  & 0 &  0  & 0 & 1  & 0 &  0  & 0 & 0\\
    0  & 0 &  0  & 0 & 1  & 0 &  0  & 0 & 0\\
    0  & 0 &  0  & 0 & 0  & 0 &  0  & 0 & 0\\
    0  & 0 &  0  & 0 & 0  & 0 &  0  & 0 & 0\\
    0  & 0 &  0  & 0 & 0  & 0 &  0  & 0 & 0\\
    0  & 0 &  0  & 0 & 0  & 0 &  0  & 0 & 0\\
    0  & 0 &  0  & 0 & 0  & 0 &  0  & 0 & 0\\
\end{bmatrix}
\]

The resulting output is a zero padded upsampled version of the input.

\begin{algorithm*}[!bth]
    \caption{Create Matrix for Single Convolutional Filter given Input with dimensions $f \times s \times s$}
    \label{alg:MatrixConvolution}
    \begin{algorithmic}[1]
    \Require{$parameters$:= parameters of $f$ trained $3 \times 3$ CNN filters, $s$:= width and height of image without zero padding,  $f$:= depth of image, $stride$:= stride of CNN filter} 
    \Ensure{Matrix $C$ representing convolution operation}
    \Statex
    \Function{CreateFilterMatrix}{$parameters, s, f, stride$}
        \State {$paddedSize \gets s + 2$}
        \State {$resized \gets s / stride$}
        \State {$rowBlocks \gets $ zeros matrix size $(f, (paddedSize)^2)$}
        \For{$filterIndex \gets 0$ to $f - 1$}
            \For {$kernelIndex \gets 0$ to $8$}
                \State{$rowIndex \gets kernelIndex \mod{3} + paddedSize *  \floor{ \frac{kernelIndex}{3}} $}
                \State {$rowBlocks[filterIndex][rowIndex] \gets parameters[filterIndex][kernelIndex]$}
            \EndFor
        \EndFor
        \State {$C \gets $ zeros matrix of size ($(resized + 2)^2, f \cdot paddedSize^2 $)}
        \State {$index \gets resized + 2 + 1$}
        \For{$shift \gets 0$ to $resized - 1$}                    
         \State {$nextBlock \gets$ zeros matrix of size $(resized, f \cdot paddedSize^2)$}
         \State {$nextBlock[0] \gets rowBlocks$}
         \For{$rowShift \gets 1$ to $resized - 1$} 
             \State {$nextBlock[rowShift] \gets rowBlocks$ shifted right by $stride \cdot rowShift$}
         \EndFor
         \State {$C[index:index + resized, :] \gets nextBlock$}
         \State {$index \gets index + resize + 2$}
         \State {$rowBlock \gets$ zero shift $rowBlock$ by $paddedSize \cdot stride$}
        \EndFor
        \State \Return {$C$}
    \EndFunction    
    \end{algorithmic}
\end{algorithm*}

\begin{algorithm*}[!bth]
    \caption{Create Matrix for Nearest Neighbor Upsampling Layer}
    \label{alg:MatrixUpsampling}
    \begin{algorithmic}[1]
    \Require{$s$:= width and height of image without zero padding,  $f$:= depth of image, $scale$:= re-scaling factor for incoming image} 
    \Ensure{Matrix $U$ representing convolution operation}
    \Statex
    \Function{CreateUpsamplingMatrix}{$s, f, scale$}
        \State {$outputSize \gets s \cdot scale + 2$}
        \State {$U \gets$ zeros matrix of size $(f \cdot outputSize^2, f \cdot (s + 2)^2)$}
        \State {$index \gets outputSize + 1$}
        \For {$filterIndex \gets 0$ to $f - 1$}
            \For {$rowIndex \gets 1$ to $s$}
                \For {$scaleIndex \gets 0$ to $scale - 1$}
                    \For{$columnIndex \gets 0 $ to $s$}
                        \State{$row \gets $ zeros vector of size $(f(s+2)^2)$}
                        \State{$row[columnIndex + rowIndex(s + 2) +  filterIndex(s + 2)^2] \gets 1$}
                        \For{$repeatIndex \gets 0$ to $scale - 1$}
                            \State{$U[index] \gets row$}
                            \State{$index \gets index + 1$}
                        \EndFor
                    \EndFor
                    \State{$index \gets index + 2$}
                \EndFor
            \EndFor
            \State{$index \gets index + 2 \cdot outputSize$}
        \EndFor
        
        \State \Return {$U$}
    \EndFunction    
    \end{algorithmic}
\end{algorithm*}

\section{Deep Linear Convolutional Autoencoders Memorize}
\label{sec:LinearConvMemorization}


While Theorem 4 provided a lower bound on the depth required for memorization, Table~\ref{tbl:Filters-on-Memorization} shows that the depth predicted by this bound is not sufficient.  In each experiment, we trained a linear convolutional autoencoder to encode $2$ randomly sampled images of size $3 \times 3$ with a varying number of layers and filters per layer.  The first $3$ rows of Table~\ref{tbl:Filters-on-Memorization} show that the lower bound from Theorem 4 is not sufficient for memorization (regardless of overparameterization through filters) since memorization would be indicated by a rank $2$ solution (with the third eigenvalue close to zero).  In fact, the remaining rows of Table~\ref{tbl:Filters-on-Memorization} show that even $8$ layers are not sufficient for memorizing two images of size $3 \times 3$.

\begin{table}[!t]
\scriptsize
\centering
\begin{tabular}{|>{\centering}m{.5cm}|>{\centering}m{.5cm}|>{\centering}m{.8cm}|>{\centering}m{.7cm}|>{\centering}m{.7cm}|>{\centering}m{.7cm}|>{\centering\arraybackslash}m{1.4cm}|}
\hline
Image Size & \# Train. Ex. & Heuristic Lower Bound Layers & \# of Layers & \# of Filters Per Layer & \# of Params  & Spectrum \\
\hline 
$3 \times 3$ & 2 & 9 & 2 & 1 & 27  & $1, 1, .98, \ldots$ \\
\hline
$3 \times 3$ & 2 & 9 & 2 & 16 & 2592  & $1, 1, .99, \ldots$ \\
\hline
$3 \times 3$ & 2 & 9 & 2 & 128 & 149760  & $1, 1, .99,  \ldots$ \\
\Xhline{3\arrayrulewidth}
$3 \times 3$ & 2 & 9 & 5 & 1 & 45  & $1, 1, .78,  \ldots$ \\
\hline
$3 \times 3$ & 2 & 9 & 5 & 16 & 7200  & $1, 1, .72,  \ldots$ \\
\hline
$3 \times 3$ & 2 & 9 & 5 & 128 & 444672  & $1, 1, .7,  \ldots$ \\
\Xhline{3\arrayrulewidth}
$3 \times 3$ & 2 & 9 & 8 & 1 & 72  & $1, 1, .14,  \ldots$ \\
\hline
$3 \times 3$ & 2 & 9 & 8 & 16 & 14112  & $1, 1, .1,  \ldots$ \\
\hline
$3 \times 3$ & 2 & 9 & 8 & 128 & 887040 & $1, 1, .08,  \ldots$ \\
\hline
\end{tabular}
\caption{Linear convolutional autoencoders with a varying number of layers and filters were initialized close to zero and trained on $2$ normally distributed images of size $3 \times 3$. Memorization does not occur in any of the examples (memorization would be indicated by the spectrum containing two eigenvalues that are $1$ and the remaining eigenvalues being close to 0). Increasing the number of filters per layer has minimal effect on the spectrum.}  
\vspace{-0.3cm}
\label{tbl:Filters-on-Memorization}
\end{table} 

Next we provide a heuristic bound to determine the depth needed to observe memorization (denoted by ``Heuristic Lower Bound Layers'' in Tables~\ref{tbl:Filters-on-Memorization} and \ref{tbl:Single-Filter-Depth-Memorization}). Theorem 4 and Table~\ref{tbl:Filters-on-Memorization} suggest that the number of filters per layer does not have an effect on the rank of the learned solution.  We thus only consider networks with a single filter per layer with kernel size $3$.  It follows from Section 1 of the main text that overparameterized single layer fully connected autoencoders memorize training examples when initialized at $0$.  Hence, we can obtain a heuristic bound on the depth needed to observe memorization in linear convolutional autoencoders with a single filter per layer based on the number of layers needed for the network to have as many parameters as a fully connected network.  The number of parameters in a single layer fully connected linear network operating on vectorized images of size $s \times s$ is $s^4$.  Hence,  using a single filter per layer with kernel size $3$, the network needs $\lceil \frac{s^4}{9} \rceil$ layers to achieve the same number of parameters as a fully connected network.  This leads to a heuristic lower bound of $\lceil \frac{s^4}{9} \rceil$ layers for memorization in  linear convolutional autoencoders  operating on images of size $s \times s$.  

 \begin{table}[!t]
\scriptsize
\centering
\begin{tabular}{|>{\centering}m{.5cm}|>{\centering}m{0.5cm}|>{\centering}m{0.8cm}|>{\centering}m{.7cm}|>{\centering}m{.6cm}|>{\centering\arraybackslash}m{2.5cm}|}
\hline
Image Size & \# of Train. Ex. & Heuristic Lower Bound Layers & \# of Layers & \# of Params & Spectrum \\
\hline 
$2 \times 2$ & 1 & 2 & 3   & 27 &  $1, [< 10^{-2}]$ \\
\hline
$3 \times 3$ & 1 & 9  & 9   &  81 & $1, [< 10^{-2}]$ \\
\hline
$4 \times 4$ & 1 & 29  & 29   &  261  & $1, [< 10^{-3}]$ \\
\hline
$5 \times 5$ & 1 & 70 & 70 &  630  & $1, [< 10^{-5}]$ \\
\hline
$6 \times 6$  & 1 & 144 & 144*  &  1296 & $1, [< 2 \cdot 10^{-2}]$ \\
\hline
$7 \times 7$ &  1 & 267 & 267*  &  2403  & $1, [< 4 \cdot 10^{-3}]$ \\
\Xhline{3\arrayrulewidth}
$3 \times 3$ & 3 & 9 & 10  &  90  & $1, .98, .98, [< 10^{-2}]$ \\
\hline
$5 \times 5$ & 5 & 70 & 200* &  1800 & $1, 1, 1, 1, 1, [< 10^{-3}]$ \\
\hline
$7 \times 7$ & 5 & 144 & 350* & 3105 & $1, 1, 1, 1, 1, [< 10^{-2}]$ \\
\hline
\end{tabular}
\caption{Linear convolutional autoencoders with $\lceil \frac{s^4}{9} \rceil$ layers (as predicted by our heuristic lower bound) with a single filter per layer, initialized with each parameter as close to zero as possible, memorize training examples of size $s\times s$ similar to a single layer fully connected system.
The bracket notation in the spectrum indicates that the magnitude of the remaining eigenvalues in the spectrum is below the value in the brackets.  }
\label{tbl:Single-Filter-Depth-Memorization}
\vspace{-0.3cm}
\end{table}

In Table \ref{tbl:Single-Filter-Depth-Memorization}, we investigate the memorization properties of networks that are initialized with parameters as close to zero as possible with the number of layers given by our heuristic lower bound and one filter of kernel size $3$ per layer.  The first 6 rows of the table show that all networks satisfying our heuristic lower bound have memorized a single training example since the spectrum consists of a single eigenvalue that is $1$ and remaining eigenvalues with magnitude less than $\approx 10^{-2}$.   Similarly, the spectra in the last 3 rows indicate that networks satisfying our heuristic lower bound also memorize multiple training examples, thereby suggesting that our bound is relevant in practice. 

The experimental setup was as follows: All networks were trained using gradient descent with a learning rate of $10^{-1}$, until the loss became less than $10^{-6}$ (to speed up training, we used Adam \cite{Adam} with a learning rate of $10^{-4}$ when the depth of the network was greater than $10$). For large networks with over 100 layers (indicated by an asterisk in Table \ref{tbl:Single-Filter-Depth-Memorization}), we used skip connections between every 10 layers, as explained in \cite{ResNet}, to ensure that the gradients can propagate to earlier layers.  Table \ref{tbl:Single-Filter-Depth-Memorization} shows the resulting spectrum for each experiment, where the eigenvalues were sorted by there magnitudes.  The bracket notation indicates that all the remaining eigenvalues have magnitude less than the value provided in the brackets.  Interestingly, our heuristic lower bound also seems to work for deep networks that have skip connections, which are commonly used in practice.


The experiments in Table~\ref{tbl:Single-Filter-Depth-Memorization} indicate that over $200$ layers are needed for memorization of $7 \times 7$ images.  In the next section, we discuss how downsampling can be used to construct much smaller convolutional autoencoders that memorize training examples.

\section{Role of Downsampling for Memorization in Convolutional Autoencoders}
\label{sec:RoleOfDownsampling}

\begin{figure}[!t]
    \centering
    \begin{subfigure}[t]{0.44\textwidth}
        \centering
        \includegraphics[height=.7in]{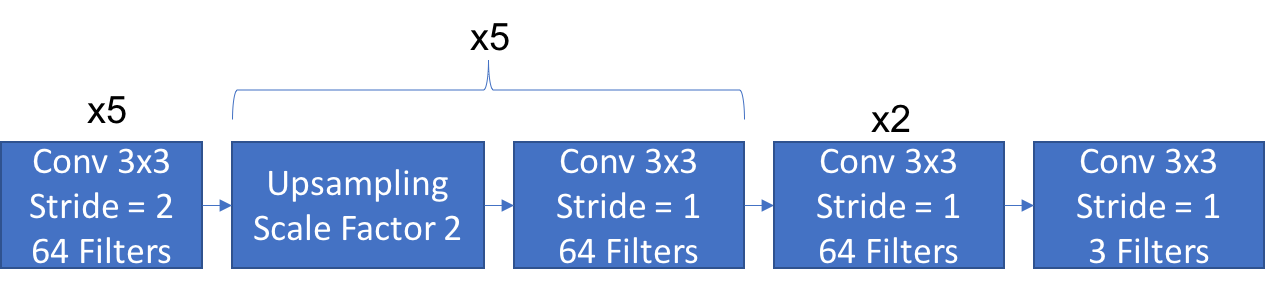}
        \vspace{-0.5cm}
        \caption{Downsampling Network Architecture D1} 
        \label{fig:D2}
        \vspace{0.5cm}
    \end{subfigure} 
    \begin{subfigure}[t]{0.48\textwidth}
        \centering
        \includegraphics[height=1.8in]{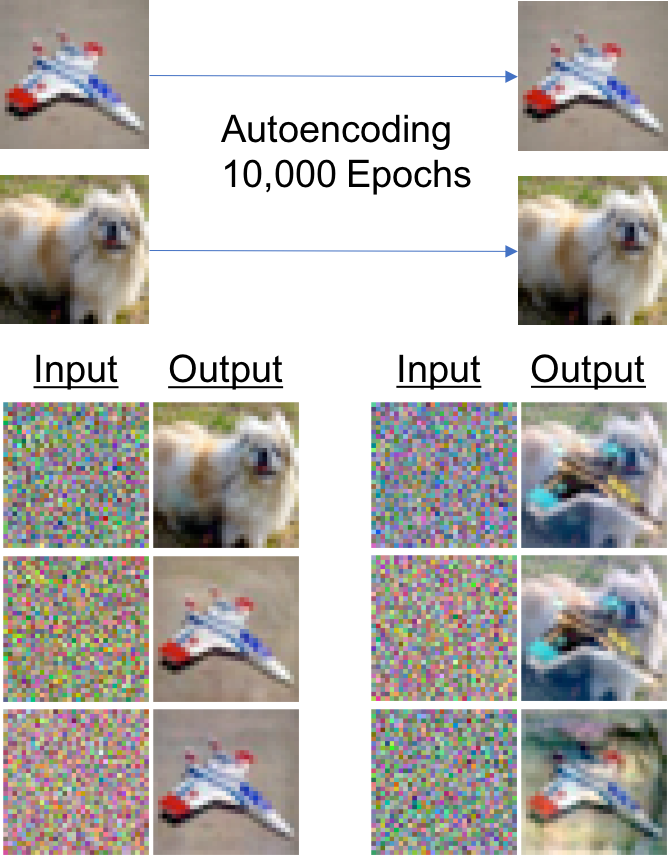}
        \caption{Network D1 trained on a CIFAR10 Dog.}
        \label{fig:D2Example}
    \end{subfigure}
    \caption{Downsampling Network Architecture D1 trained on a two images from CIFAR10 for 10000 iterations.  Each parameter of the network is initialized using the default PyTorch initialization.  When fed standard Gaussian data the model outputs a linear combination of training images since the corresponding solution has spectrum $1, 1, [< 2 \cdot 10^{-2}]$.}
    \label{full_fig}
\end{figure}
\vspace{-0.3cm}
\begin{figure}[!t]
    \centering
    \begin{subfigure}[t]{0.44\textwidth}
        \centering
        \includegraphics[height=.5in]{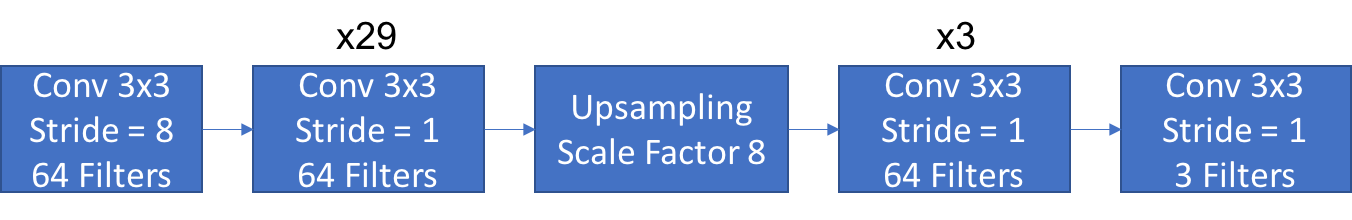}
            \vspace{-0.6cm}
        \caption{Downsampling Network Architecture D2} 
        \label{fig:D1}
        \vspace{0.4cm}
    \end{subfigure} 
    \begin{subfigure}[t]{0.48\textwidth}
        \centering
        \includegraphics[height=1.3in]{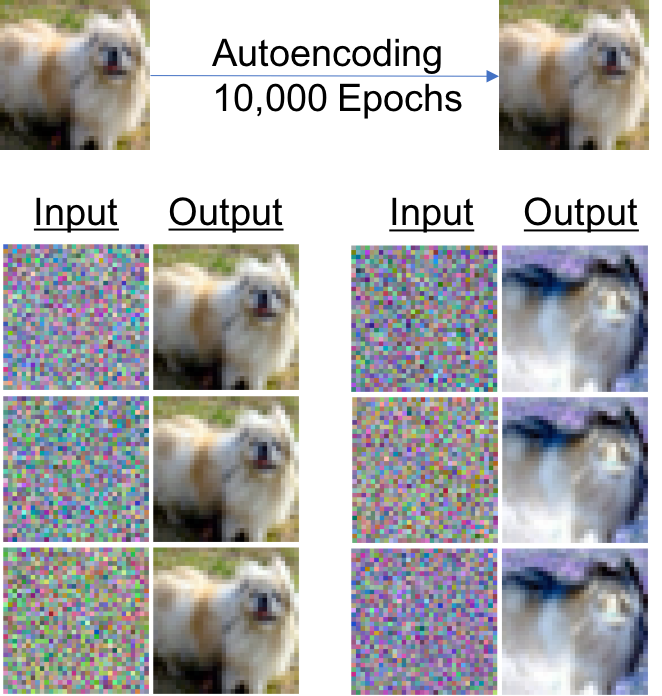}
        \caption{Network D2 trained on a CIFAR10 Dog.}
        \label{fig:D1Example}
    \end{subfigure}
    \caption{Downsampling Network Architecture D2 trained on a single image from CIFAR10 for 10000 iterations.  Each parameter of the network is initialized at $1.8\cdot 10^{-3}$.  When fed standard Gaussian data the model outputs a multiple of the training image since the corresponding solution has spectrum $1, [<7 \cdot 10^{-6}]$.}
\end{figure}


To gain intuition for why downsampling can trade off depth to achieve memorization, consider a convolutional autoencoder that downsamples input to $1 \times 1$ representations through non-unit strides. Such extreme downsampling makes a convolutional autoencoder equivalent to a fully connected network; hence given the results in Section 1 in the main text, such downsampling convolutional networks are expected to memorize. This is illustrated in Figure \ref{full_fig}: The network uses strides of size $2$ to progressively downsample to a $ 1 \times 1$ representation of a CIFAR10 input image.  Training the network on two images from CIFAR10, the rank of the learned solution is exactly $2$ with the top eigenvalues being $1$ and the corresponding eigenvectors being linear combinations of the training images.  In this case, using the default PyTorch initialization was sufficient in forcing each parameter to be close to zero.

Memorization using convolutional autoencoders is also observed with less extreme forms of downsampling.  In fact, we observed that downsampling to a smaller representation and then operating on the downsampled representation with depth provided by our heuristic bound established in Section~\ref{sec:LinearConvMemorization} also leads to memorization. As an example, consider the network in Figure \ref{fig:D1} operating on images from CIFAR10 (size $32 \times 32$).  This network downsamples a $32 \times 32$ CIFAR10 image to a $4 \times 4$ representation after layer~$1$. As suggested by our heuristic lower bound for $4\times 4$ images (see Table \ref{tbl:Single-Filter-Depth-Memorization}) we use 29 layers in the network.   Figure~\ref{fig:D1Example} indicates  that this network 
indeed memorized the image by producing a solution of rank $1$ with eigenvalue $1$ and corresponding eigenvector being the dog image.  


\section{Superattractors in Nonlinear Convolutional Autoencoders}
\label{sec:strong-mem-nonlinear}
We start by investigating whether the heuristic bound on depth needed for memorization that we have established for linear convolutional autoencoders carries over to nonlinear convolutional autoencoders.

\begin{figure}[!b]
    \centering
    \includegraphics[height=1.5in]{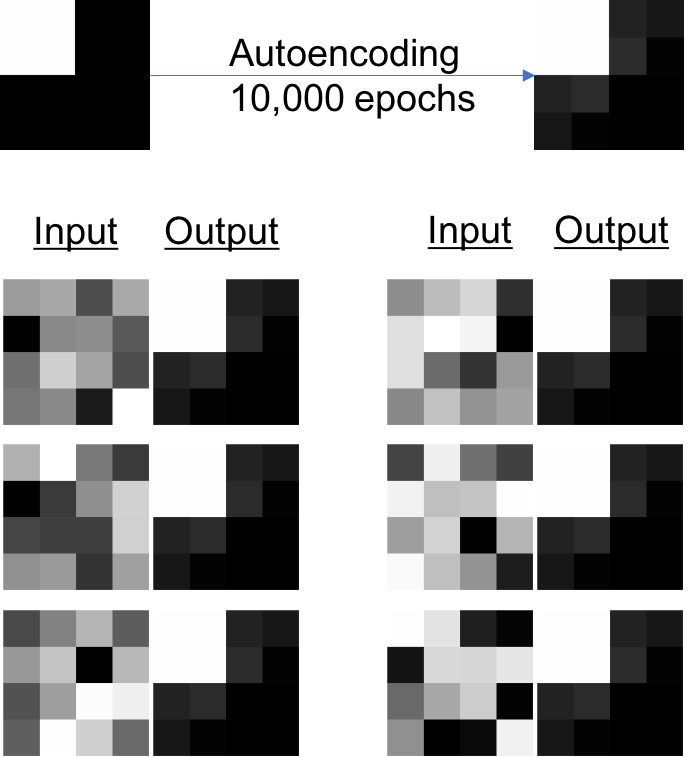}
    \caption{A $29$ layer network with a single filter of kernel size 3, 1 unit of zero padding, and stride 1 followed by a leaky ReLU activation per layer initialized with every parameter set to $10^{-1}$ memorizes $4 \times 4$ images.  Our training image consists of a white square in the upper left hand corner and the test examples contain pixels drawn from a standard normal distribution.}
    \label{fig:NonlinearNondownsampling29Layers}
    \vspace{-0.3cm}
\end{figure}

\begin{figure}[!bt]
    \centering
    \includegraphics[height=1.5in]{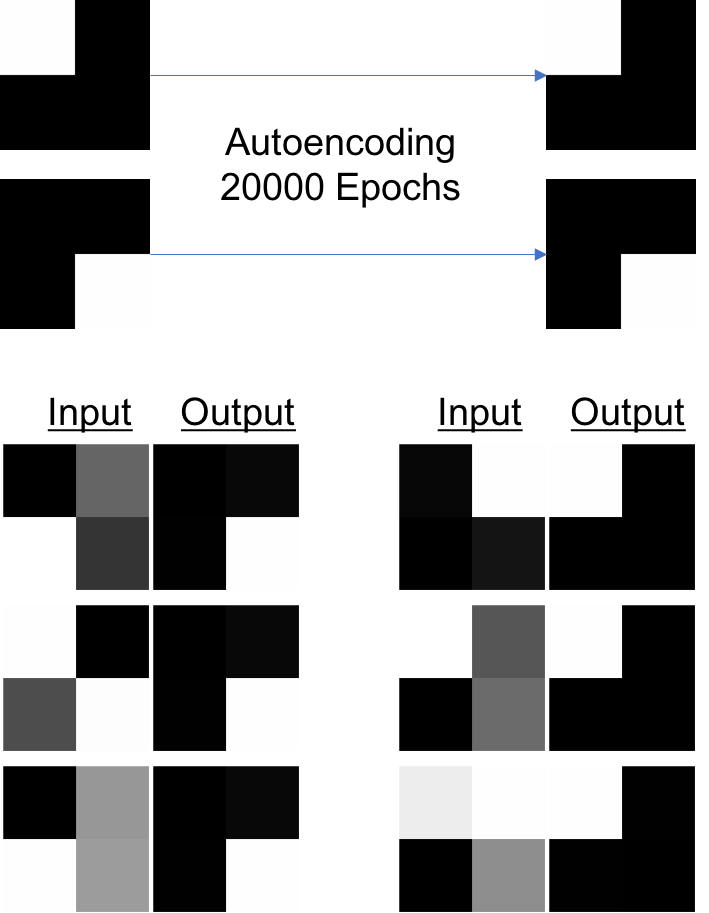}
    \caption{A $5$ layer nonlinear network strongly memorizes $2 \times 2$ images.  The network has a single filter of kernel size 3, 1 unit of zero padding, and stride 1 followed by a leaky ReLU activation per layer with Xavier Uniform initialization. The network also has skip connections between every 2 layers.  The training images are orthogonal: one with a white square in the upper left corner and one with a white square in the lower right corner. The test examples contain pixels drawn from a standard normal distribution.}
    \label{fig:NonlinearNondownsampling5Layers}
    \vspace{-0.3cm}
\end{figure}

\begin{example}
Consider a deep nonlinear convolutional autoencoder with a single filter per layer of kernel size $3$, $1$ unit of zero padding, and stride $1$ followed by a leaky ReLU \cite{LeakyReLU} activation that is initialized with parameters as close to 0 as possible.  In Table~\ref{tbl:Single-Filter-Depth-Memorization} we reported that its linear counterpart memorizes $4\times 4$ images with 29 layers. Figure \ref{fig:NonlinearNondownsampling29Layers} shows that also the corresponding nonlinear network with 29 layers can memorize $4 \times 4$ images. 
 While the spectrum can be used to prove memorization in the linear setting, since we are unable to extract a nonlinear equivalent of the spectrum for these networks, we can only provide evidence for memorization by visual inspection.  
\end{example}

This example suggests that our results on depth required for memorization in deep linear convolutional autoencoders carry over to the nonlinear setting. In fact, when training on multiple examples, we observe that memorization is of a stronger form in the nonlinear case.  Consider the example in Figure~ \ref{fig:NonlinearNondownsampling5Layers}. We see that  given new test examples, a nonlinear convolutional autoencoder with $5$ layers trained on $2 \times 2$ images outputs \emph{individual} training examples instead of combinations of training examples.

\section{Robustness of Memorization}

\subsection{Contraction and overfitting}
\label{appendix:Manifold Reconstruction}
It is possible for a deep autoencoder to contract towards training examples even as it faithfully learns the identity function over the data distribution. 
The following theorem states that all training points can be memorized by a neural network while achieving an arbitrarily small expected reconstruction error.

\begin{theorem}
For any training set $\{x_i\}_{i \in [1..n]}$ and for any $\epsilon > 0$, there exists a 2-layer fully-connected autoencoder $f$ with ReLU activations and $(n+1) \cdot d$ hidden units such that (1) the expected reconstruction error loss of $f$ is less than $\epsilon$ and (2) $\{x_i\}_{i \in [1..n]}$ are attractors of the discrete dynamical system with respect to $f$.
\end{theorem}

\begin{proof}

Properties (1) and (2) can be achieved by $d$ piecewise linear functions with $n$ changepoints. First, let us consider the 1D setting. For simplicity, assume that the domain of $f$ is bounded, e.g. the support of the data distribution is a subset of $[0,1]$. Define $\delta = \frac{1}{4} \min_{x_i, x_j} | x_i - x_j|$. Assume that $x_1, \cdots, x_n$ are unique so that $\delta > 0$ and ordered such that $0 < x_1 < x_2 < \cdots < x_n < 1$. Finally, assume that $\epsilon \leq \delta$ (if not, we can consider a smaller $\epsilon = \delta$). Consider the function
\begin{equation*}
    f(x) = 
    \begin{cases}
      \frac{x_1-\delta+\epsilon}{x_1-\delta}x, & \text{if}\ x < x_1 - \delta \\
      ax + (1-a)(x_i - \delta) + \epsilon, & \text{if}\ x \in [x_i - \delta, x_i + \delta) \\
      b_ix + (1-b_i)(x_i + \delta) - \epsilon, & \text{if}\ x \in [x_i+\delta, x_{i+1} - \delta) \\
      \frac{1-(x_n+\delta-\epsilon)}{1-(x_n+\delta)}(x-1)+1, & \text{if}\ x \geq x_n+\delta
    \end{cases}
\end{equation*}
where $a = \frac{2\delta - 2\epsilon}{2 \delta}$ and $b_i = \frac{x_{i+1}-x_i-2\delta + 2\epsilon}{x_{i+1}-x_i-2\delta}$. Note that this function achieves expected reconstruction error less than $\epsilon$ and the absolute value of the slope at the training examples $\leq 1$, satisfying properties (1) and (2). Furthermore, this function is piecewise linear with $n+1$ changepoints, which can be represented by a 2-layer FC neural network with ReLU activations and $n+1$ hidden units (Theorem 2.2, Arora et al. 2018). To extend this to the $d$-dimensional setting, we can use this same construction to define $f(x)$ element-wise by $d$ piecewise linear functions, which can accordingly be represented by a 2-layer FC neural network with ReLU activations and $(n+1) \cdot d$ hidden units.

\end{proof}



A concrete visualization of contraction being uncoupled from manifold reconstruction is shown in Figure 1a, top right. In this example, the learned autoencoder has low reconstruction error, since it is close to the $f(x)=x$ line, yet it is contractive to the training examples as shown in Figure 1a, bottom left. To show in practice that this uncoupling can be achieved, we show in Figure \ref{fig:MNIST-recon} that it is possible to recover training examples from an MNIST autoencoder with near-zero reconstruction error. 

\begin{figure*}
\centering
\includegraphics[scale=0.2]{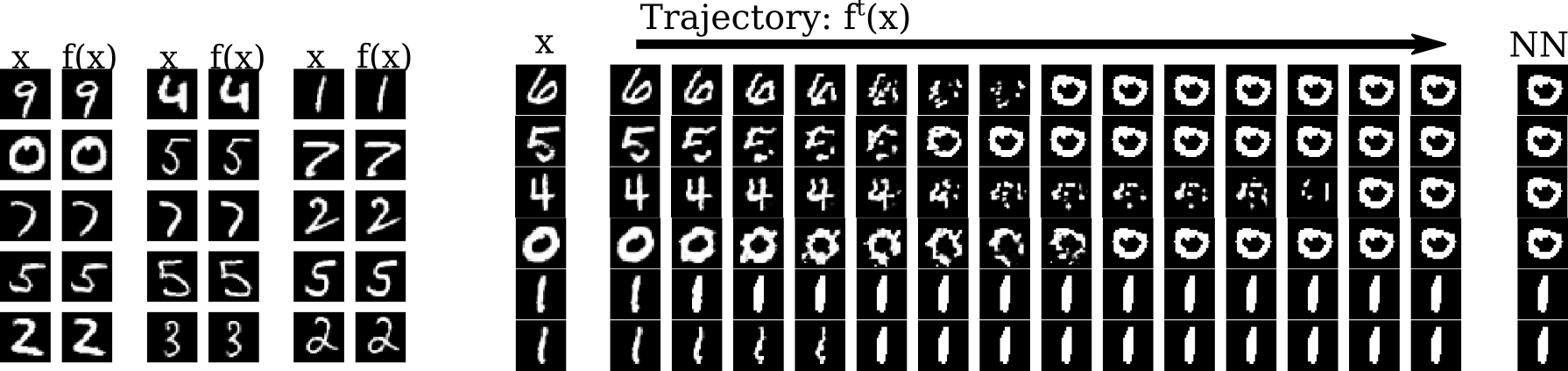}
\caption{A fully-connected autoencoder (7 layers, 512 hidden neurons) was trained to convergence on 50000 samples from the MNIST dataset and achieves near-zero reconstruction error on the test set (train MSE = 0.00117, test MSE = 0.00159). Left: Reconstructions of test images whose trajectories converge to training examples. Right: The trajectories of test examples were computed by iterating the autoencoder over the images. From left to right, the images in the trajectory represent the results after multiple iterations. }
\label{fig:MNIST-recon}
\end{figure*}

\subsection{Initialization at Zero is Necessary for Memorization}
\label{appendix:Initialization}
Section 1 in the main text showed that linear fully connected autoencoders initialized at zero memorize training examples by learning the minimum norm solution.  Since  in the linear setting the distance to the span of the training examples remains constant when minimizing the autoencoder loss regardless of the gradient descent algorithm used, non-zero initialization leads to a noisy form of memorization.  More precisely, when using non-zero initialization in the linear setting, extending Weyl's inequality to singular values yields that the singular values of the solution are bounded above by the singular values of the minimum norm solution plus the largest singular value of the initialized matrix.  Hence as long as the initialization is sufficiently small, then the solution is dominated by the singular values of the minimum norm solution and so memorization is present.  Motivated by this analysis in the linear setting, to see memorization, we require that each parameter of an autoencoder be initialized as close to zero as possible (while allowing for training).  To provide further intuition for this result, the impact of large initialization is presented in the 1D autoencoders trained in Figure 1 in the main text.  Namely in the upper left and upper right of Figure 1 in the main text, we see that the learned map with a large initialization behaves very differently from the learned map with close to zero initialization.  For example, the latter is contractive around the training examples, while the former is not.  Hence, analyses that attempt to rationalize behavior prior to training based on random matrix theory do not necessarily explain the phenomena observed after training.          

We now briefly discuss how popular initialization techniques such as Kaiming uniform/normal \cite{KaimingInit}, Xavier uniform/normal \cite{Xavier}, and default PyTorch initialization \cite{PyTorch} relate to zero initialization.  In general, we observe that Kaiming uniform/normal initialization leads to an output with a larger $\ell_2$ norm as compared to a network initialized using Xavier uniform/normal or PyTorch initializations.  Thus, we do not expect Kaiming uniform/normal initialized networks to present memorization as clearly as the other initialization schemes.  That is, for linear convolutional autoencoders, we expect these networks to converge to a solution further from the minimum nuclear norm solution and for nonlinear convolutional autoencoders, we expect these networks to produce noisy versions of the training examples when fed arbitrary inputs. This phenomenon is demonstrated experimentally in the examples in Figure~\ref{fig:Initializations}.   

\begin{figure}[!t]
    \centering
        \begin{subfigure}[t]{0.23\textwidth}
        \centering
        \includegraphics[height=.8in]{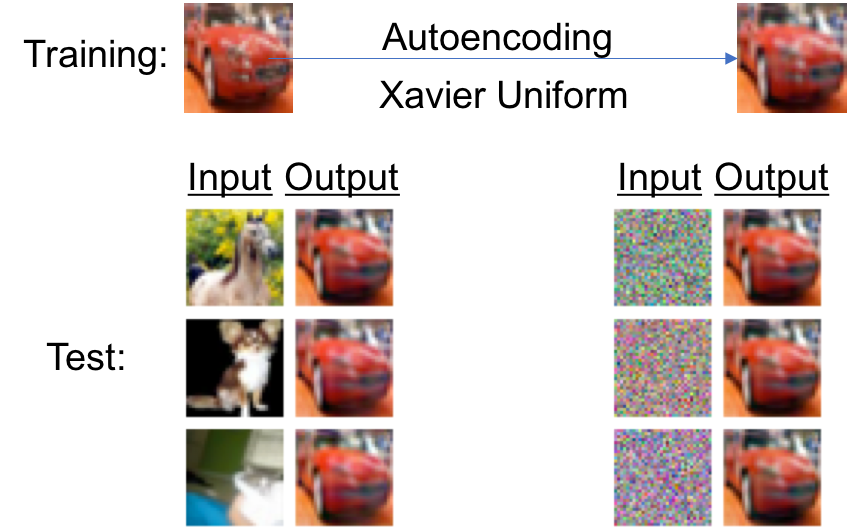}
        \caption{Output Norm: $.0087$}
        \label{fig:NonlinearXavierUniform}   
         \vspace{0.2cm}
    \end{subfigure}%
    ~ 
    \begin{subfigure}[t]{0.23\textwidth}
        \centering
        \includegraphics[height=.8in]{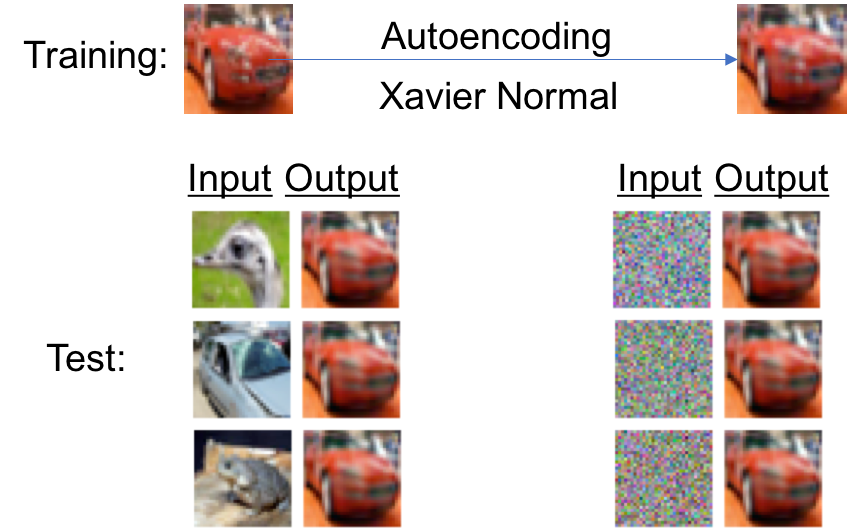}
        \caption{Output Norm: $.0079$}
        \label{fig:NonlinearXavierNormal}
         \vspace{0.2cm}
    \end{subfigure}
    ~
    \begin{subfigure}[t]{0.23\textwidth}
        \centering
        \includegraphics[height=.8in]{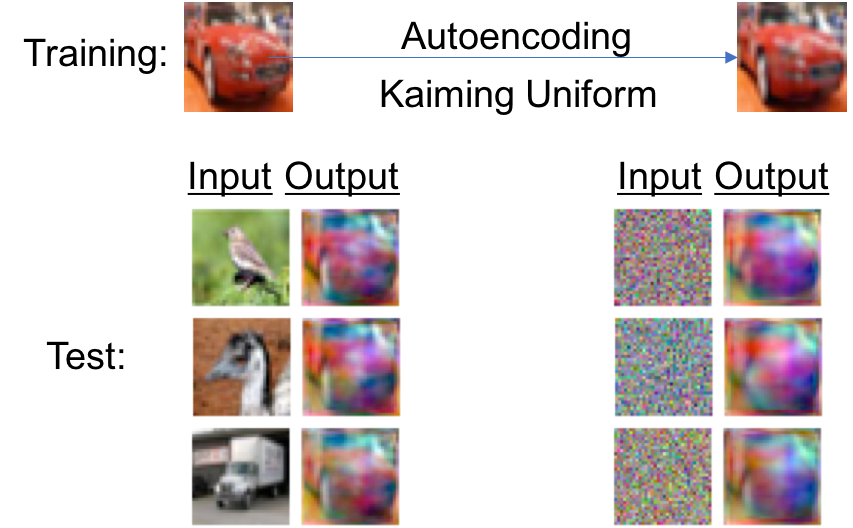}
        \caption{Output Norm: $9.67$}
        \label{fig:NonlinearKaimingUniform}
    \end{subfigure}%
    ~ 
    \begin{subfigure}[t]{0.23\textwidth}
        \centering
        \includegraphics[height=.8in]{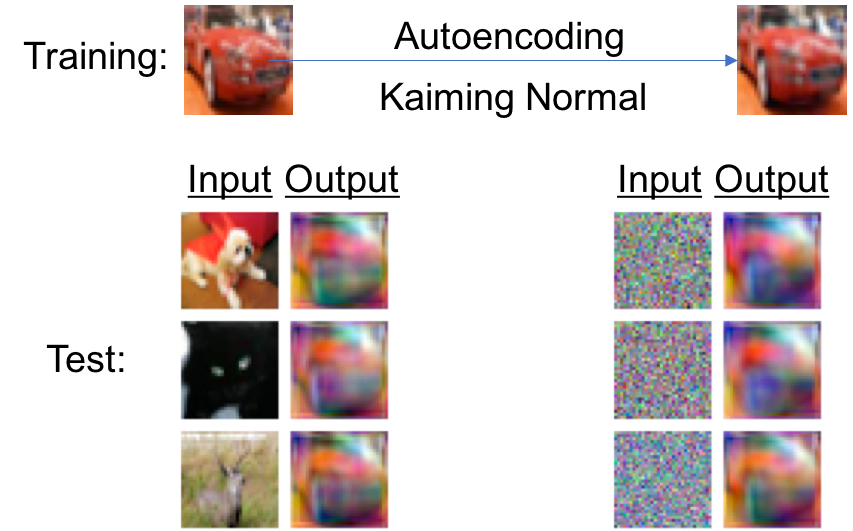}
        \caption{Output Norm: $17.64$}
        \label{fig:NonlinearKaimingNormal}
    \end{subfigure}    
    \vspace{-0.1cm}
    \caption{Effect of popular initialization strategies on memorization: Each figure demonstrates how the nonlinear version of the autoencoder from Figure 3a (modified with Leaky ReLU activations after each convolutional layer) behaves when initialized using Xavier uniform/normal and Kaiming uniform/normal strategies.  We also give the $\ell_2$ norm of the output for the training example prior to training. Consistent with our predictions, the Kaiming uniform/normal strategies have larger norms and the output for arbitrary inputs shows that memorization is noisy.}
    \label{fig:Initializations}  
    \vspace{-0.3cm}
\end{figure}

\end{document}